\documentclass[letterpaper, 10 pt, journal, twoside]{IEEEtran} 
\IEEEoverridecommandlockouts                              
\usepackage{bm}
\usepackage{enumerate}
\usepackage{commath}
\usepackage{graphicx}
\graphicspath{{Pictures/}}
\usepackage{amsmath}
\usepackage{subcaption}
\usepackage{breqn}
\usepackage{cite}
\usepackage[table]{xcolor}
\usepackage{booktabs}
\usepackage{empheq}
\usepackage{amsfonts}
\usepackage{amssymb}

\usepackage{amsthm}
\usepackage{hyperref}
\usepackage{xcolor}
\usepackage{mathrsfs}
\usepackage{accents}
\usepackage{thmtools}
\usepackage{thm-restate}
\usepackage{float}

\theoremstyle{plain}
\newtheorem{lemma}{Lemma}
\newtheorem{definition}{Definition}
\newtheorem{assumption}{Assumption}
\newtheorem{property}{Property}

\newtheorem{remark}{Remark}
\newtheorem{example}{Example}
\newtheorem{problem}{Problem}
\newtheorem*{problem*}{Problem}
\newtheorem*{theorem*}{Theorem}
\newtheorem{assumption*}{Assumption}

\declaretheorem[name=Theorem]{thm}

\newcommand{\bmb}[1]{\bar{\myvar{#1}}} 
\newcommand{\bmh}[1]{\hat{\myvar{#1}}} 
\newcommand{\bmdh}[1]{\dot{\hat{\myvar{#1}}}}

\newcommand{\hdot}[1]{\dot{\hat{#1}}}
\newcommand{\bmdd}[1]{\ddot{\myvar{#1}}} 
 
\newcommand{\ubar}[1]{\underaccent{\bar}{#1}}
\newcommand{\ubscr}[1]{\ubar{\mathscr{#1}}}

\newcommand{\myvar}[1]{\bm{#1}}
\newcommand{\myvardot}[1]{\dot{\myvar{#1}}}

\newcommand{\myset}[1]{\mathscr{#1}}
\newcommand{\robscalar}[1]{\hat{#1}}
\newcommand{\robscalardot}[1]{\dot{\hat{#1}}}
\newcommand{\robvar}[1]{\hat{\myvar{#1}}}
\newcommand{\robmat}[1]{\hat{#1}}
\newcommand{\robset}[1]{\hat{\mathscr{#1}}}
\newcommand{\robubset}[1]{\hat{\ubscr{#1}}}
\newcommand{\apvar}[1]{\bmh{#1}}
\newcommand{\apvardot}[1]{\bmdh{#1}}
\newcommand{\apmat}[1]{\hat{#1}}
\newcommand{\apmatdot}[1]{\dot{\hat{#1}}}
\newcommand{\apscalar}[1]{\hat{#1}}
\newcommand{\apscalardot}[1]{\dot{\hat{#1}}}
\newcommand{\myvarnorm}[1]{ \left\lVert \myvar{#1} \right\rVert}
\newcommand{\myvarnormex}[2]{ \left\lVert \myvar{#1} #2 \right\rVert}
\newcommand{\mysetnorm}[2]{ \myvarnorm{#1}_{\myset{#2}}}

\newcommand{\bracketmat}[2]{ \left[ \begin{array}{#1} #2 \end{array} \right] }

\title{\LARGE \bf
Control Barrier Functions for Mechanical Systems: Theory and Application to Robotic Grasping
}

\author{Wenceslao Shaw Cortez, Denny Oetomo, Chris Manzie, and Peter Choong
\thanks{W. Shaw Cortez, D. Oetomo, and C. Manzie are with the MIDAS Laboratory in the School of Electrical, Mechanical, and Infrastructure Engineering, 
	University of Melbourne, 3010, Australia,
        {\tt\small shaww@student.unimelb.edu.au, doetomo@unimelb.edu.au, manziec@unimelb.edu.au}.}
 \thanks{ P. Choong is with the Department of Surgery, University of Melbourne, St. Vincent's Hospital, 3065,
 		Australia,
 		{\tt\small  pchoong@unimelb.edu.au}.}
}

\begin{document}

\maketitle
\thispagestyle{plain}
\pagestyle{plain}

\begin{abstract}

Control barrier functions have been demonstrated to be a useful method of ensuring constraint satisfaction for a wide class of controllers, however existing results are mostly restricted to continuous time systems of relative degree one. Mechanical systems, including robots, are typically second-order systems in which the control occurs at the force/torque level. These systems have velocity \textit{and} position constraints (i.e. relative degree two) that are vital for safety and/or task execution. Additionally, mechanical systems are typically controlled digitally as sampled-data systems. The contribution of this work is two-fold. First, is the development of novel, robust control barrier functions that ensure constraint satisfaction for relative degree two, sampled-data systems in the presence of model uncertainty. Second, is the application of the proposed method to the challenging problem of robotic grasping in which a robotic hand must ensure an object remains inside the grasp while manipulating  it to a desired reference trajectory. A grasp constraint satisfying controller is proposed that can admit existing nominal manipulation controllers from the literature, while simultaneously ensuring no slip, no over-extension (e.g. singular configurations), and no rolling off of the fingertips. Simulation and experimental results validate the proposed control for the robotic hand application.
%

\end{abstract}

\section{Introduction}

Mechanical systems, such as robots, comprise many of the engineered systems used in the physical world. Their dynamics are conventionally defined as second-order, nonlinear affine systems, with generalized states associated with position, velocity, and acceleration. The control of these systems is usually applied as a force/torque at the acceleration level, and most, if not all, mechanical systems must satisfy workspace constraints. Examples include collision avoidance of autonomous vehicles in traffic \cite{Funke2017}, a robotic manipulator avoiding obstacles in the environment \cite{Rauscher2016}, and of particular note here is a robotic hand manipulating an object \cite{ShawCortez2018b}. Workspace constraints are position dependent, which due to the second order dynamics, makes them relative degree two. Additionally, these mechanical systems are controlled digitally. That is, at each sampling period, sensors collect measurements of the system and a control input is implemented constantly over the sampling time. Thus mechanical systems can be equivalently defined as sampled-data systems of relative degree two.

An existing approach for formally handling constraints is via control barrier functions. Control barrier functions can be classified as reciprocal control barrier functions and zeroing control barrier functions \cite{Ames2017}. The former was developed with applications towards bipedal walking \cite{Hsu2015}, systems evolving on manifolds \cite{Wu2015a}, and control of constrained robots \cite{Rauscher2016}. However the latter, zeroing control barrier functions, have been shown to not only be more practical for implementation, but also robust to model uncertainties \cite{Ames2017, Xu2015a}. 

Zeroing control barrier functions are attractive due to their robustness qualities, however existing zeroing control barrier functions are mostly restricted to continuous-time, relative degree one systems. The implementation of those methods is conventionally done by formulating the control input as a quadratic program, and requires local Lipschitz continuity of the ensuing control input \cite{Ames2017}. That implementation requires restrictive properties, namely linear independent constraint qualification and complimentary slackness, to ensure local Lipschitz continuity of the quadratic program \cite{Morris2015, Robinson1982}. Those requirements are difficult to guarantee in practice when multiple constraints must be considered for mechanical systems. Furthermore the continuous-time formulation does not allow for robustness to sampling effects that arise during implementation on mechanical systems. Only \cite{Ghaffari2018} has addressed (reciprocal) control barrier function for sampled data systems, however that approach was restricted to double integrator systems, which does not apply to the mechanical systems considered here, nor does it consider robustness to perturbations.

An additional concern is that mechanical systems have a physical relationship between position, velocity, and actuation effort, which is not considered in existing methods for relative degree two systems. This relationship is apparent as the position of a mechanical system approaches the constraint boundary. If there is no consideration of the position/velocity relationship, then a large velocity may be permitted near the constraint boundary. Due to the inertia of the system, large control effort is then required to keep the system inside the constraint set. This large control effort may lead to actuator saturation, which compromises the systems ability of ensuring constraint satisfaction. This issue becomes exacerbated when considering external perturbations on the system that may arise from model uncertainty commonly found in practice. Thus it is important to systematically consider velocity profiles in the zeroing control barrier function formulation, which has not yet been addressed in existing work \cite{Ames2017, Xu2015a}.

%

The first contribution of this paper is the development of a novel zeroing control barrier function for mechanical systems. The proposed method formally ensures constraint satisfaction for relative degree two systems and is robust to sampling effects and perturbations. The proposed method also allows for the designer to tune the velocity bounds to reduce excessive actuator effort near the constraint boundary and help comply with actuator limitations of the mechanical system. A constraint satisfying controller is then defined to stay minimally close to a nominal control, but prioritize constraint satisfaction. The second contribution is the application of the proposed technique to a non-trivial problem: robotic grasping.

Robotic grasping consists of a robotic hand manipulating an object to a desired reference pose trajectory. This task is commonly referred to as in-hand manipulation. In-hand manipulation consists of moving an object to track a desired trajectory, while simultaneously ensuring the object remains within the grasp. For successful in-hand manipulation, it is paramount to guarantee that the object remains in the grasp during the manipulation motion. 
%

A failed grasp can result from slipping, joint over-extension, and excessive rolling. Slipping is an obvious grasping concern, which has been  extensively addressed in the literature \cite{ShawCortez2018b, Caldas2015, Kerr1986}. Joint over-extension relates to joints exceeding feasible joint angles (e.g joint workspace and singular hand configurations), which inhibits the robotic hand from applying necessary contact forces on the object \cite{Murray1994}. Excessive rolling is when the contact points roll off of the fingertip surface. In-hand manipulation inherently relies on rolling motion for object manipulation \cite{Montana1988}. However excessive rolling motion may cause the contact points to leave the fingertip surface, resulting in loss of contact with the object. Thus for successful manipulation, the object must not slip, the joints must remain inside a feasible workspace, and the contact points must remain in the fingertip workspace. These conditions are referred to as the grasp constraints.

To date, there exist an abundance of object manipulation controllers for robotic hands. Early work developed manipulation controllers based on linearization or feedback linearization for an exact object model \cite{Jen1996, Cole1989}. In \cite{Cheah1998} a more robust manipulation controller was developed to handle gravity, uncertain contact locations/kinematics. Later work developed passivity-based controllers \cite{Song2012}, impedance-based controllers \cite{Wimbock2012}, and visual/tactile-based controllers \cite{Jara2014}. Other bio-inspired methods have also been developed by extracting human information to perform manipulation on robotic hands \cite{Gabiccini2013,Prattichizzo2013,Farnioli2013,Colasanto2013,Salvietti2013,Gioioso2013a}. A more extensive review of object manipulation control methods can be found in \cite{Ozawa2017}. However, those existing techniques are only valid if the grasp constraints are satisfied.

There exist few methods that address grasp constraint satisfaction. In grasp force optimization, an optimization problem is solved in-the-loop in which the constraints ensure slip is prevented. Early work in grasp force optimization primarily focused on computational speed due to limited computational capabilities of the time \cite{Kerr1986, Nahon1992, Buss1996, Han2000a, Helmke2002, Boyd2007}. Later work addressed robustness properties of grasp force optimization \cite{Fungtammasan2012}, and considered the dynamics of the hand-object system \cite{Caldas2015, Fan2017, ShawCortez2018b}. However those methods only address a subset of the grasp constraint satisfaction problem, that of no slip, and still required unfounded assumptions that singular configurations were avoided and excessive rolling did not occur.

Existing methods of addressing all grasping constraints are via motion planning approaches. Motion planning approaches typically assume the exact knowledge of the hand-object model and assume the system is quasi-static to search for feasible manipulation trajectories \cite{Cherif1999, Han2000, Hertkorn2013}. In \cite{CorralesRamon2013}, the knowledge of the object is relaxed by requiring a mesh of the object surface, but still assumes a quasi-static nature of the system. Those methods rely on quasi-static assumptions that generally do not hold in a dynamic manipulation setting. Furthermore those approaches are subject to large computational resources, which may not be conducive to real-time applications \cite{Hertkorn2013}. 

Furthermore, existing methods of addressing grasping constraints are not robust to model uncertainties and/or unknown disturbances that may act on the system. The literature contains numerous controllers that aim to extend manipulation capabilities outside of a laboratory setting by restricting the required sensors to on-board sensing modalities \cite{ShawCortez2018b, Ozawa2017}. A common theme in related work is referred to as ``blind grasping" or similarly ``tactile-based blind grasping" where the robotic hand only has access to sensors that can be physically integrated into the hand (i.e. joint angle sensors and tactile sensors). Those approaches still depend on satisfaction of the grasp constraint, but with the added restriction to on-board sensors and no a priori knowledge of the object model. This motivates the need for a robust, active method of ensuring grasp constraint satisfaction that does not require exact knowledge of the object model. 

The second contribution of this paper is a novel grasp controller to actively ensure grasp constraint satisfaction. The proposed control is designed to be implemented alongside existing manipulation controllers found in the literature. The idea here is to support the many existing methods from the literature such that in the case when no constraint violation occurs, the proposed control outputs the original manipulation controller from the literature. However, should constraint violation be imminent, the proposed control deviates from the manipulation controller to prioritize constraint satisfaction (i.e. no slip, no over-extension, no excessive rolling). The proposed control is implemented in simulation and hardware to demonstrate the efficacy of the proposed approach. Note this paper extends the contribution from \cite{ShawCortez2018c} with application to sampled-data systems, robustness to model uncertainties, associated formal proofs, and implementation on hardware.

\subsection*{Notation}

Throughout this paper, an indexed vector $\myvar{v}_i \in \mathbb{R}^p$ has an associated concatenated vector $\myvar{v} \in \mathbb{R}^{pk}$, where the index $i$ is specifically used to index over the $n$ contact points in the grasp. The notation $\myvar{v}^{\mathcal{\mathcal{E}}}$ indicates that the vector $\myvar{v}$ is written with respect to a frame $\mathcal{E}$, and if there is no explicit frame defined, $\myvar{v}$ is written with respect to the inertial frame, $\mathcal{P}$. The operator $(\cdot)\times$ denotes the skew-symmetric matrix representation of the cross-product. $SO(3)$ denotes the special orthogonal group of dimension 3. The $r\times r$ identity matrix is denoted $I_{r\times r}$. The term $\myvar{e}_j \in \mathbb{R}^{1,r}$ denotes the $j$th row of $I_{r\times r}$. The Lie derivatives of a function $h(\myvar{x})$ for the system $\myvardot{x} = \myvar{f}(\myvar{x}) + g(\myvar{x}) \myvar{u}$ are denoted by $L_f h$ and $L_gh$, respectively. The intersection of sets $\mathscr{C}_j$ for $j\in [1,l]$ is denoted $\ubscr{C}$. The distance between a point $\myvar{x}$ and set $\myset{A}$ is $\mysetnorm{x}{A} := \underset{\myvar{y} \in \myset{A}}{\text{inf}}$ $\myvarnorm{x - y}$. A $\mathcal{KL}$ function as defined in \cite{Khalil2002} is denoted by $\mathcal{KL}(\cdot, \cdot)$. When discussing model uncertainty, the approximation of a variable $\myvar{v}$ is denoted with a hat, $\apvar{v}$, and the associated error is denoted by $\Delta(\myvar{v})$.

\section{Zeroing Control Barrier Functions for Sampled-Data Systems of Relative Degree Two}\label{sec:control barrier fcns}

In this section, the zeroing control barrier functions from \cite{Ames2017} are extended to relative degree two, sampled-data systems and address robustness to perturbations.

\subsection{Robust Barrier Functions for Relative Degree Two Systems}

Consider the following nonlinear affine control system in continuous time:
\begin{equation}\label{eq:nonlinear affine dynamics}
\myvardot{x} = \myvar{f}(\myvar{x}) + \myvar{g}(\myvar{x}) \myvar{u} + \myvar{d}
\end{equation}
where $\myvar{u} \in U \subseteq \mathbb{R}^m$ is the control input, $\myvar{f}, \myvar{g}$ are locally Lipschitz continuous functions of $\myvar{x} \in \mathbb{R}^p$, and $\myvar{d} \in \mathbb{R}^p$ is a bounded, locally Lipschitz disturbance. Let $\myvar{x}(t,\myvar{x}_0) \in \mathbb{R}^n$ be the solution of \eqref{eq:nonlinear affine dynamics}, which for ease of notation is denoted by $\myvar{x}$. 

The goal of constraint satisfaction is to ensure the states $\myvar{x}$ stay within a set of constraint-admissible states. Let $h_j(\myvar{x}): \mathbb{R}^p \to \mathbb{R}$ be a twice-continuously differentiable, relative degree two function for constraint $j \in [1,l]$. Robustness is addressed using margins, $\delta_j \in \mathbb{R}_{\geq 0}$ in $h_j$, which leads to the following definition of $\robscalar{h}_j$:
\begin{equation}\label{eq:robust h}
\robscalar{h}_j(\myvar{x}) = h_j(\myvar{x}) - \delta_j, \ j\in [1,l]
\end{equation}

Let the set of constraint-admissible states be:
\begin{equation}\label{eq:constraint set multiple}
\begin{split}
\robset{C}_j = \{ \myvar{x} \in \mathbb{R}^p: \robscalar{h}_j (\myvar{x}) \geq 0 \}, \  j \in [1,l] \\
\partial \robset{C}_j =  \{ \myvar{x} \in \mathbb{R}^p: \robscalar{h}_j(\myvar{x}) = 0 \}, \  j \in [1,l] \\
\text{Int}( \robset{C}_j)=  \{ \myvar{x} \in \mathbb{R}^p: \robscalar{h}_j(\myvar{x}) > 0 \}, \  j \in [1,l] \\
\end{split}
\end{equation}
For ease of notation,  $\robubset{C}$ represents the intersection of all $\robset{C}_j$ for $j \in [1,l]$. Also, $\myset{C}_j$ is used to denote $\robset{C}_j$ with $\delta_j \equiv 0$.

The following definition for extended class-$\mathcal{K}$ function is now introduced:
\begin{definition} \label{def:extended class K} \cite{Ames2017}:
A continuous function, $\alpha:(-b,a) \to (-\infty,\infty)$ for $a,b \in \mathbb{R}_{>0}$ is an \textit{extended class-$\mathcal{K}$ function} if it is strictly increasing and $\alpha(0) = 0$.
\end{definition} 
\noindent Note for clarity, the extended class-$K$ functions addressed here will be defined for $a,b = \infty$. 

Constraint satisfaction is ensured by showing that on the constraint boundary, the system states are directed into the interior or along the boundary of the constraint set \cite{Blanchini2015}. This is equivalent to guaranteeing that $\robscalardot{h}(\myvar{x}) \geq - \alpha_1(\robscalar{h}(\myvar{x}))$ for a continuously differentiable, extended class-$\mathcal{K}$ function $\alpha_1$ \cite{Ames2017}. 

The novel approach taken here is to introduce a continuously differentiable function $\robscalar{B}_j: \mathbb{R}^p \to \mathbb{R}$ defined by:
\begin{equation}\label{eq:candidate zeroing cbf multiple}
\robscalar{B}_j(\myvar{x}) = \robscalardot{h}_j(\myvar{x}) + \alpha_1 (\robscalar{h}_j(\myvar{x})) - \beta_j, \  j \in [1,l]
\end{equation} 
where $\beta_j \in \mathbb{R}_{\geq 0}$ is a robustness margin. By construction of $\robscalar{B}_j$, constraint satisfaction regarding $\robset{C}_j$ is equivalent to ensuring $\robscalar{B}_j \geq 0$ for all $t\geq 0$. Let $\robset{B}_j$ denote the set where $\robscalar{B}_j \geq 0$:
\begin{equation}\label{eq:set B for zbf}
\robset{B}_j =  \{ \myvar{x} \in \mathbb{R}^p: \robscalar{B}_j(\myvar{x}) \geq 0 \}
\end{equation}
As with $\robscalar{h}_j$ and $\robset{C}_j$, $B_j$ and $\myset{B}_j$ will be used to denote $\robscalar{B}_j$, $\robset{B}_j$ for $\beta_j \equiv 0$. To implement the control barrier functions, let $\robset{S}_{u_j}$ denote the set of constraint-admissable control inputs:
\begin{multline}\label{eq:Su for zcbf}
\robset{S}_{u_j}(\myvar{x}) = \{ \myvar{u} \in U: \\ L_f \robscalar{B}_j(\myvar{x}) + L_g \robscalar{B}_j(\myvar{x}) \myvar{u}  + \alpha_2(\robscalar{B}_j(\myvar{x})) \geq  0 \}, \ j \in [1,l]
\end{multline}

By ensuring forward invariance of $\robset{B}_j$, it follows that $\robscalar{B}_j \geq 0$ to ensure that $\robset{C}_j$ is forward invariant. The function $\robscalar{B}_j$ is considered the \textit{zeroing control barrier function}. Here, input-to-state stability properties of zeroing control barrier functions are exploited to ensure forward invariance of $\ubscr{C}$ in the presence of bounded perturbations. This is accomplished by using robustness margins $\beta_j$, $\delta_j$ for compact sets $\ubscr{B}$, $\ubscr{C}$. Forward invariance of $\ubscr{C}$ for the proposed method is formally guaranteed in the following theorem:

\begin{thm}\label{thm:ZBF robust}
Consider the controllable system \eqref{eq:nonlinear affine dynamics}. Let $\robset{C}_j$ be defined by  \eqref{eq:robust h}, \eqref{eq:constraint set multiple}, for the twice-continuously differentiable, relative degree two functions $h_j:\mathscr{D}_j \to \mathbb{R}$, defined on the open set $\myset{D}_j \supset \myset{C}_j \supset \robset{C}_j$. Suppose there exist a continuously differentiable extended class-$\mathcal{K}$ function $\alpha_1$ and extended class-$\mathcal{K}$ function $\alpha_2$ such that for $\robscalar{B}_j: \mathscr{E}_j \to \mathbb{R}$ (defined by \eqref{eq:candidate zeroing cbf multiple}) on the open set $\myset{E}_j \supset\robset{B}_j$ (defined by \eqref{eq:set B for zbf}), and $\robubset{S}_{u}$ (defined by \eqref{eq:Su for zcbf}), $\robubset{B}$ and $\robubset{C}$ are compact, and $\robubset{B} \cap \robubset{C}$, $\robubset{S}_u$ are non-empty. Then for any $\myvar{x}(0) \in \robubset{B} \bigcap \robubset{C}$, and locally Lipschitz control $\myvar{u}(\myvar{x}) \in \ubscr{S}_u(\myvar{x})$, there exist $\theta, \beta_j, \delta_j \in \mathbb{R}_{\geq 0}$, $j\in[1,l]$ such that for $\myvarnorm{d}_{\infty} \leq \theta$, $\myvar{x}$ remains in $\ubscr{B} \bigcap \ubscr{C}$ for $t \geq 0$.
\end{thm}
\begin{proof}
First, asymptotic stability of $\robubset{B}$ is considered for $\myvar{d} \equiv 0$. Let $V_{\robubset{B}}: \mathbb{R}^p  \to \mathbb{R}_{\geq 0}$ be the Lyapunov function defined by $V_{\robubset{B}} = \sum_j V_{\robset{B}_j}$ on the open set $\ubscr{E}$ where $V_{\robset{B}_j}$ is defined by:
\begin{equation}
V_{\robset{B}_j}(\myvar{x}) = \begin{cases}
0, & \text{if} \ \myvar{x} \in \robset{B}_j \\
-\robscalar{B}_j(\myvar{x}),  & \text{if} \ \myvar{x} \in \myset{E}_j \setminus \robset{B}_j
 \end{cases}
\end{equation} 
To show negative definiteness of $\dot{V}_{\robubset{B}}$, note that for $\myvar{x} \in  \myset{E}_j \setminus \robset{B}_j $,  $\robscalar{B}_j(\myvar{x})  < 0$. For $\myvar{u} \in \robset{S}_{u_j}$, differentiation of $V_{\robset{B}_j}$ for $\myvar{x} \in \myset{E}_j \setminus \robset{B}_j$ results in $\dot{V}_{\robset{B}_j} = -\robscalardot{B}_j \leq \alpha_2(\robscalar{B}_j) = \alpha_2(-V_{\robset{B}_j}) < 0$ for each $j \in [1,l]$. Thus $\dot{V}_{\robubset{B}} < 0$ holds by definitino of $V_{\robubset{B}}$ and asymptotic stability of the compact set $\robubset{B}$ follows from Theorem 2.8 of \cite{Lin1996}. Smoothness of the Lyapunov function is addressed in Proposition 4.2 of \cite{Lin1996}. 

Next, asymptotic stability of $\robubset{C}$ is addressed. Consider the following Lyapunov function, $V_{\robubset{C}}: \mathbb{R}^p \to \mathbb{R}_{\geq 0}$, defined by $V_{\robubset{C}} = \sum_j V_{\robset{C}_j}$ on the open set $\ubscr{D}$ where $V_{\robset{C}_j}$ is defined by:
\begin{equation}
V_{\robset{C}_j} (\myvar{x}) = \begin{cases}
0, & \text{if} \ \myvar{x} \in \robset{C}_j \\
-\robscalar{h}_j(\myvar{x}), & \text{if} \ \myvar{x} \in \myset{D}_j  \setminus   \robset{C}_j
 \end{cases}
\end{equation} 
From $\myvar{x} \in \robset{B}_j$, it follows that $\robscalar{B}_j \geq 0$ holds. Thus $V_{\robset{C}_j}(\myvar{x}) = 0$ for $\myvar{x} \in \robset{C}_j$, $V_{\robset{C}_j}(\myvar{x}) >0$ for $\myvar{x} \in  \myset{D}_j\setminus \robset{C}_j  $. Differentiation of $V_{\robset{C}_j}$ for $\myvar{x} \in \myset{D}_j \setminus \robset{C}_j$ results in $\dot{V}_{\robset{C}_j} = -\robscalardot{h}_j \leq \alpha_1(\robscalar{h}_j) = \alpha_1(-V_{\robset{C}_j}) < 0$. Similarly with $\robubset{B}$, asymptotic stability of $\robubset{C}$ follows.  

Now consider $\myvar{d} \neq 0$. It follows from Proposition 5 of \cite{Xu2015a} that for a class-$\mathcal{K}$ function $\gamma$, the set:
 \begin{equation}
 \myset{B}_{\gamma_j} = \{ \myvar{x} \in \mathbb{R}^p: \apscalar{B}_j(\myvar{x}) \geq -\gamma(\myvarnorm{d}_{\infty}) \}
 \end{equation}
is asymptotically stable for $ \myvarnorm{d}_{\infty} \leq \theta$ such that  $\myvarnorm{x}_{\mathscr{B}_\gamma} \leq \mathcal{KL}(\myvarnormex{x}{(0)}_{\myset{B}_\gamma}, t)$  for all $t \geq 0$. Choose $\beta_j = \gamma( \myvarnorm{d}_{\infty})$. By \eqref{eq:candidate zeroing cbf multiple}, it follows that $B_j = \robscalar{B}_j + \beta_j$ and consequently $\myset{B}_j = \myset{B}_{\gamma_j}$. Thus $\myset{B}_j$ is asymptotically stable such that $\myvarnorm{x}_{\myset{B}_j} \leq \mathcal{KL}( \myvarnormex{x}{(0)}_{\myset{B}_j} , t)$. By the condition that $\myvar{x}(0) \in \robset{B}_j \subset \myset{B}_j$, it follows that $ \myvarnorm{x}_{\myset{B}_j} = 0$ for all $t \geq 0$, and thus  $\myvar{x}$ remains in $\myset{B}_j$ for $t \geq 0$.  
 
Forward invariance of $\myset{B}_j$ implies that $\myset{C}_j$ is asymptotically stable. Thus by the same argument, it follows that there exists a $\delta_j $ such that $\myvar{x}$ remains in $\myset{C}_j$ for $t \geq 0$. Since $\myvar{x}$ remains in $\myset{B}_j$ and $\myset{C}_j$ concurrently, $\myvar{x}$ must remain in $\myset{B}_j \bigcap \myset{C}_j$ for $t \geq 0$. Repeated application for $j \in [1,l]$ completes the proof.
\end{proof}

Theorem \ref{thm:ZBF robust} ensures constraint satisfaction via the proposed zeroing control barrier functions in the presence of disturbances. The development of the proposed zeroing control barrier functions is also advantageous for implementation on mechanical systems of relative degree two. The following example highlights this advantage by demonstrating how the propose method naturally restricts large velocities from occurring at the constraint boundary:

\begin{example}\label{ex:velocity bounds}
One advantage of the zeroing control barrier functions presented here over existing methods \cite{Ames2017} is that it naturally constrains the velocity of the relative-degree two system for desired behavior near the constraint boundary. This is particularly advantageous for mechanical systems with workspace constraints. Consider the following system:
\begin{equation}
\dot{x}_1 = x_2
\end{equation}
\begin{equation}
\dot{x}_2 = f(\myvar{x}) +g(\myvar{x}) \myvar{u}
\end{equation}
where $\myvar{x} = (x_1, x_2) \in \mathbb{R}^2 $, $\myvar{u} \in \mathbb{R}^2$. The system is constrained by: $1 \leq x_1 \leq 4$.

Let $h_{\text{min}}(x) = x - 1$, $h_{\text{max}} = 4 - x$. From \eqref{eq:candidate zeroing cbf multiple}, the resulting control barrier functions are $B_{\text{min}} = x_2 + \alpha_1(x-1)$ and $B_{\text{max}} = -x_2 + \alpha_1(4-x)$ for a given extended class-$\mathcal{K}$ function, $\alpha_1$. The zeroing control barrier functions naturally bound the velocity, $x_2$, with respect to the extended class-$\mathcal{K}$ function $\alpha_1$: 
\begin{equation}
\alpha_1(x_1-1) \leq x_2 \leq \alpha_1(4-x_1) 
\end{equation}

The bounds on $x_2$ are shown in Figure \ref{fig:velocity bounds} for various choices of $\alpha_1$. Figure \ref{fig:velocity bounds} shows that without additional structure imposed on the proposed control barrier functions, the velocity bounds prevent large velocities at the boundary that would require large control effort to ensure constraint satisfaction. Furthermore, the designer has the ability to tune these velocity bounds by an appropriate choice of $\alpha_1$. The gray regions in Figure \ref{fig:velocity bounds} depict the set of constraint-admissible states $\myvar{x}$.

\begin{figure}[hbtp]
\centering
	\subcaptionbox{$\alpha_1(h) = h$ \label{fig:velbound_linear} }
		{\includegraphics[scale=.2]{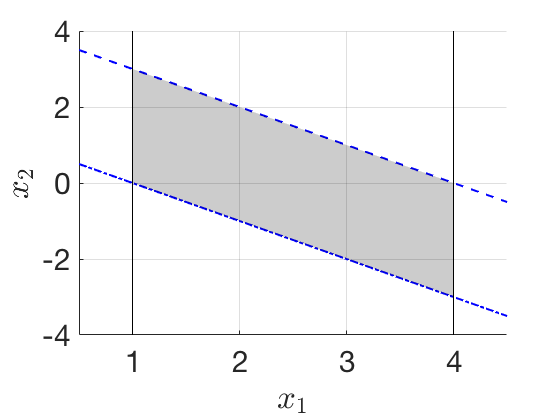}}
	\subcaptionbox{$\alpha_1(h) = 0.15 h^3$  \label{fig:velbound_cubic} }
		{\includegraphics[scale=.2]{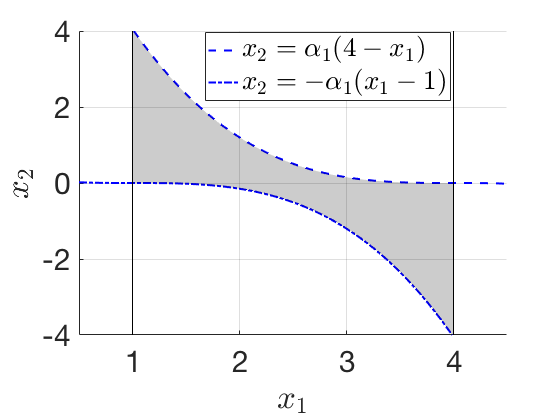}}
	\subcaptionbox{$\alpha_1(h) = 2 \tan^{-1}(h)$ \label{fig:velbound_atan} }
		{\includegraphics[scale=.2]{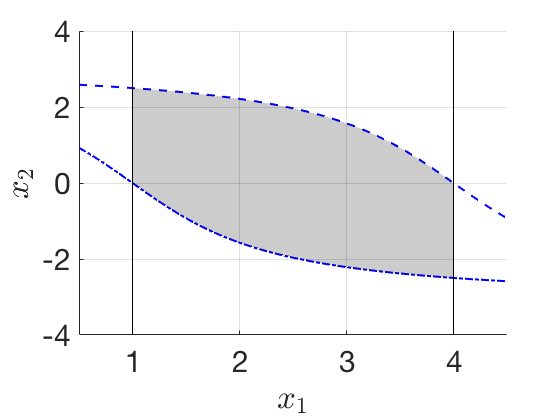}}
	
	\caption{Induced velocity bounds for various choices of $\alpha_1$. Note the shaded regions depict the constraint-admissible set for $x_1, x_2$.} \label{fig:velocity bounds}
\end{figure}

\end{example}
%

\begin{remark}
It is worth mentioning that the proposed zeroing control barrier function method is extendable to larger relative degrees by repeated applications of \eqref{eq:candidate zeroing cbf multiple}. For example, for a relative degree 3 system, a new function $\apscalar{E}_j = \apscalardot{B}_j + \alpha_3(\apscalar{B}_j) - \eta_j$, with $\eta_j \in \mathbb{R}_{\geq 0}$, would be defined as the zeroing control barrier function and the same results would follow.
\end{remark}

\subsection{Control Barrier Functions for Sampled-Data Systems}

Theorem \ref{thm:ZBF robust} ensures constraint satisfaction for some locally Lipschitz continuous $\myvar{u} \in \robubset{S}_u$. However insofar there is no discussion of how such a $\myvar{u}$ is to be constructed. In related work \cite{Ames2017, Rauscher2016}, the control $\myvar{u}$ is defined as the solution to a quadratic program, which takes the following form for the zeroing control barrier functions presented here:
\begin{align} \label{eq:consat qp robust ct}
\begin{split}
\myvar{u}^* \hspace{0.1cm} = \hspace{0.1cm} & \underset{\myvar{u} \in \mathbb{R}^m}{\text{argmin}}
\hspace{.1cm} \myvar{u}^T \myvar{u}  - 2 \myvar{u}_{\text{nom}}^T \myvar{u}  \\
& \text{s.t.}  \hspace{0.6cm} \apmat{A}(\myvar{x}) \myvar{u} \geq \apvar{b}(\myvar{x})  \\
& \hspace{0.7cm}  \myvar{u}_{\text{min}} \leq \myvar{u} \leq \myvar{u}_{\text{max}}
\end{split}
\end{align}
where $\myvar{u}_{\text{nom}} \in \mathbb{R}^m$ is a nominal control and
\begin{equation}
\robmat{A}(\myvar{x}) = \bracketmat{ccc}{L_g \robscalar{B}_1(\myvar{x})^T, & ... ,& L_g \robscalar{B}_l(\myvar{x})^T}^T
\end{equation}
\begin{equation}
\robvar{b}(\myvar{x}) = \bracketmat{c}{-L_f \robscalar{B}_1(\myvar{x}) - \alpha_2(\robscalar{B}_1(\myvar{x})) \\ ... \\ -L_f \robscalar{B}_l(\myvar{x}) - \alpha_2( \robscalar{B}_l(\myvar{x}))}
\end{equation} 

The idea behind the quadratic program formulation is to simplify the control design by first designing a stabilizing, nominal control for the system without consideration of constraint satisfaction. Then the quadratic program \eqref{eq:consat qp robust ct} ensures constraint satisfaction as the nominal control stabilizes the system to a desired point. This implementation prioritizes constraint satisfaction over the stability/performance guarantees associated with the nominal controller. It is important to emphasize that existing control barrier function technqiues require local Lipschitz continuity of $\myvar{u}^*$ \cite{Ames2017, Rauscher2016}. However the quadratic program formulation alone is not sufficient to ensure local Lipschitz continuity as discussed in \cite{Robinson1982}. To ensure local Lipschitzness of the quadratic program \eqref{eq:consat qp robust ct}, the following properties must hold \cite{Fiacco1976, Morris2015}:
\begin{property}\label{prop:linear independence}
(Linear Independence Constraint Qualification) The active constraints of the quadratic program \eqref{eq:consat qp robust ct} have full row rank.
\end{property}

\begin{property}\label{prop:complimentary slackness}
(Strict Complimentary Slackness) Let $\lambda^* \in \mathbb{R}^m$ denote the Lagrange multiplier associated with $\myvar{u}^*$. Strict complimentary slackness is satisfied if there does not exist any $j$ such that both $\lambda_j^* = 0$ and $A_j \myvar{u}_j^* = 0$. ($A_j$ refers to the $j$th row of $A$).
\end{property}
Note that Properties \ref{prop:linear independence} and \ref{prop:complimentary slackness} can be overly restrictive and difficult to ensure. Linear independence in particular can be troublesome when there are more constraints than decision variables in the optimization problem leading to potential redundant, active constraints. This occurs in robotic grasping, as will be discussed later, in which slip and workspace constraints outnumber the control input of the system.

The quadratic program formulation from \eqref{eq:consat qp robust ct} is valuable for mechanical systems, particularly robotic systems, in which constraint satisfaction is paramount for safety and successful execution of a task. However \eqref{eq:consat qp robust ct} is not directly implementable for mechanical systems. Mechanical systems are typically controlled digitally, in which sensors are used to take measurements, the quadratic program \eqref{eq:consat qp robust ct} is solved, and the computed control is implemented over a sampling period. This sample-hold operation is a common characteristic in most mechanical systems, yet is neglected in existing control barrier function methods.

Here, sampling effects are considered to extend the proposed zeroing control barrier function method for implementation on mechanical systems. Consider the sampled-data version of \eqref{eq:nonlinear affine dynamics}, where a zero-order hold is placed on $\myvar{u}$:
\begin{equation}\label{eq:sampled data system}
\myvardot{x} = \myvar{f}(\myvar{x}) + \myvar{g}(\myvar{x}) \myvar{u}_k + \myvar{d}
\end{equation}
and $\myvar{u}_k \in \mathbb{R}^m$ is a piece-wise constant control with sampling time of $T \in \mathbb{R}_{>0}$. The sampled-data system takes measurements of $\myvar{x}_k := \myvar{x}(kT)$, which are constant in each sampling period. Uniqueness and existence of the solution $\myvar{x}(t, \myvar{x}(0))$ for \eqref{eq:sampled data system} is ensured by Caratheodory's Theorem for a time interval $t \in [0, NT)$, $N \in \mathbb{Z}_{>0}$ \cite{Grune2011}.

The extension of zeroing control barrier functions to sampled-data systems is as follows. First, inter-sampling behavior is addressed by incorporating a margin $\nu(T)$ in the zeroing control barrier function condition \eqref{eq:Su for zcbf}. The margin $\nu(T)$ is an extended class-$\mathcal{K}$ function, which acts to negate inter-sampling effects. Let the set of constraint-admissible controls $\myvar{u}_k$ for sampled-data systems be defined by:
\begin{multline}\label{eq:Su for zcbf sampled-data}
\robset{S}_{u_{k_j}}(\myvar{x}) = \{ \myvar{u}_k \in U: k \in [0,N]: \\ L_f \robscalar{B}_j(\myvar{x}_k) + L_g \robscalar{B}_j(\myvar{x}_k) \myvar{u}_k  + \alpha_2(\robscalar{B}_j(\myvar{x}_k)) \geq  \nu(T)\}, j \in [1,l]
\end{multline}

Following Nagumo's Theorem \cite{Blanchini2015}, the use of $\nu(T)$ is to ensure $\robscalardot{B}_j \geq - \alpha(\robscalar{B}_j)$, which ensures asymptotic stability of $\robset{B}_j$ (see Theorem \ref{thm:ZBF robust}) for forward invariance of $\myset{B}_j$. The following theorem ensures forward invariance of $\ubscr{B} \cap \ubscr{C}$ in the presence of sampling and external perturbations:

\begin{thm}\label{thm:sampled data zcbf}
Consider the controllable system \eqref{eq:sampled data system}. Let $\robset{C}_j$ be defined by  \eqref{eq:robust h}, \eqref{eq:constraint set multiple} for the thrice-continuously differentiable, relative degree two functions $h_j(\myvar{x}): \myset{D}_j \to \mathbb{R}$, defined on the open set $\myset{D}_j \supset \myset{C}_j \supset \robset{C}_j$. Suppose there exists a twice-continuously differentiable, extended class-$\mathcal{K}$ function $\alpha_1$ and locally Lipschitz continuous extended class-$\mathcal{K}$ function $\alpha_2$, such that for $\robscalar{B}_j: \mathscr{E}_j \to \mathbb{R}$ (defined by \eqref{eq:candidate zeroing cbf multiple}) on the open set $\myset{E}_j \supset\robset{B}_j$ (defined by \eqref{eq:set B for zbf}), $\robubset{B}$ and $\robubset{C}$ are compact, and $\robubset{B} \cap \robubset{C}$ is non-empty. Then for a given $T\in \mathbb{R}_{>0}$, $N \in \mathbb{Z}_{>0}$, there exists an extended class-$\mathcal{K}$ function $\nu(T)$, and $\theta$, $\beta_j$, $\delta_j \in \mathbb{R}_{\geq 0}$, $j\in [1,l]$ such that for all $\myvar{x}(0) \in \robubset{B} \bigcap \robubset{C}$, $\myvarnorm{d}_{\infty} \leq \theta$, and any bounded, piece-wise constant control $\myvar{u}_k(\myvar{x}) \in \robubset{S}_{u_k}(\myvar{x})$ (defined by \eqref{eq:Su for zcbf sampled-data}), $\myvar{x}$ remains in $\myset{B} \bigcap \myset{C}$ for $t \in [0, NT)$.
\end{thm}
\begin{proof}
First, under the local Lipschitz continuity of $\myvar{f}$, $\myvar{g}$, and $\myvar{d}$, Caratheodory's Theorem \cite{Grune2011} ensures existence and uniqueness of the solution to \eqref{eq:sampled data system} for the bounded piece-wise constant control $\myvar{u}_k$ on the interval $t \in [0, NT)$.

Next, the extended class-$\mathcal{K}$ function $\nu(T)$ is constructed. Let $m$ be defined by:
\begin{multline*}
m = \Big(L_f \robscalar{B}_j(\myvar{x}) - L_f \robscalar{B}_j(\myvar{x}_k) \Big) + \Big( \alpha_2(\robscalar{B}_j(\myvar{x})) - \alpha_2(\robscalar{B}_j(\myvar{x}_k)) \Big) \\
+ \Big(L_g \robscalar{B}_j(\myvar{x}) - L_g \robscalar{B}_j(\myvar{x}_k) \Big) \myvar{u}_k , j \in [1,l]
\end{multline*}
By locally Lipschitz properties of $\myvar{f}(\myvar{x})$, $\myvar{g}(\myvar{x})$, $\robscalar{B}_j(\myvar{x})$, and $\alpha_2(\robscalar{B}_j(\myvar{x}))$ it follows that:
\begin{multline*}
||m|| \leq (c_1  + c_2+ c_3|| \myvar{u}_k || )  || \myvar{x} - \myvar{x}_k ||,   \\ \forall t \in [kT, (k+1)T), j \in [1,l]
\end{multline*}
where $c_1, c_2, c_3 \in \mathbb{R}_{>0}$ are the respective Lipschitz constants for $L_f \robscalar{B}_j$, $\alpha_2$, and $L_g \robscalar{B}_j$.
By closeness of solutions between $\myvar{x}$ and $\myvar{x}_k$ (Theorem 3.4 of \cite{Khalil2002}) and boundedness of $\myvar{u}_k$ for $||\myvar{u}_k|| \leq c_4$, $c_4 \in \mathbb{R}_{>0}$, it follows that 
\begin{multline*}
||m|| \leq \frac{(c_1 + c_2 + c_3 c_4) c_5}{c_1 + c_2 c_4} (e^{(c_1 + c_2 c_4)(t - kT)} - 1), \\  \ \forall t \in [kT, (k+1)T), j \in [1,l]
\end{multline*}
where $c_5 \geq || \myvar{f}(\myvar{x}) + \myvar{g}(\myvar{x}) \myvar{u}_k ||$. Note $c_5$ is guaranteed to exist over the interval $t \in [kT, (k+1)T)$ due to local Lipschitz continuity of $\myvar{f}$, $\myvar{g}$ and boundedness of $\myvar{u}_k$. Let $\nu_{k_j}(T)$ be defined by:
\begin{equation}\label{eq:sample data margin}
\nu_{k_j}(T)  :=  \frac{a }{b} (e^{bT} - 1), \ \forall t \in [kT, (k+1)T), j \in [1,l]
\end{equation}
where $a = (c_1+c_2 +c_3 c_4) c_5$ and $b =c_1 + c_2 c_4$. Let $\nu(T) := \underset{k,j}{\text{max}}\ \nu_{k_j}(T)$ such that $\nu(T) \geq \nu_{k_j}(T)$ for all $j \in [1,l], k \in [0, N]$.

From $\myvar{u}_k \in \robubset{S}_{u_k}$, it follows that:
\begin{multline}\label{eq:zcbf condition for sampled data systems}
L_f \robscalar{B}_j(\myvar{x}_k) + L_g \robscalar{B}_j(\myvar{x}_k) \myvar{u}_k  + \alpha_2(\robscalar{B}_j(\myvar{x}_k)) \geq  \nu(T), \\
\forall t \in [kT, (k+1)T), \forall j \in [1,l]
\end{multline}
Addition of $L_f \robscalar{B}_j(\myvar{x}) + L_g \robscalar{B}_j(\myvar{x}) \myvar{u}_k  + \alpha_2(\robscalar{B}_j(\myvar{x})) + L_d \robscalar{B}_j(\myvar{x})$ to each side of \eqref{eq:zcbf condition for sampled data systems} results in:
\begin{multline*}
L_f \robscalar{B}_j(\myvar{x}) + L_g \robscalar{B}_j(\myvar{x}) \myvar{u}_k  + \alpha_2(\robscalar{B}_j(\myvar{x})) + L_d \robscalar{B}_j(\myvar{x}) \\ \geq \nu(T) + m + L_d \robscalar{B}_j(\myvar{x})
\end{multline*}
From the previous derivation of $\nu(T)$, it follows that $\nu(T) \geq m$ such that the following holds:
\begin{equation*}
\robscalardot{B}_j(\myvar{x}) + \alpha_2(\robscalar{B}_j(\myvar{x})) \geq L_d \robscalar{B}_j(\myvar{x})
\end{equation*}

Next, the robustness margin, $\beta_j$ is used to account for the perturbation associated with $\myvar{d}$. It is trivial to show that there exists extended class-$\mathcal{K}$ functions $\alpha_3$, $\alpha_4$ such that $\alpha_2(\robscalar{B}_j) \leq \alpha_3(B_j) + \alpha_4(-\beta_j)$, which results in:
\begin{equation*}
\robscalardot{B}_j(\myvar{x}) + \alpha_3(B_j(\myvar{x})) \geq - \alpha_4(-\beta_j) + L_d \robscalar{B}_j(\myvar{x})
\end{equation*}
Choose $\beta_j \geq \alpha_4^{-1} \circ c_6 \theta$, where $\frac{\partial \robscalar{B}}{\partial \myvar{x}} \leq c_6$, such that $-\alpha_4(-\beta_j) \geq L_d \robscalar{B}_j(\myvar{x})$. From $\robscalar{B}_j = B_j - \beta_j$, it is straightforward to see that $\robscalardot{B}_j = \dot{B}_j$, and consequently $\dot{B}_j \geq - \alpha_3(B_j(\myvar{x}))$. According to Nagumo's theorem \cite{Blanchini2015}, $\myset{B}_j$ is forward invariant on $t \in [kT, (k+1)T)$. Repeated application for $j \in [1,l]$, $t \in [0, NT)$ ensures forward invariance of $\ubscr{B}$ for $t \in [0, NT)$. The same procedure is repeated for $\robscalar{h}_j$ and $\delta_j$, which concurrently ensures forward invariance of $\ubscr{C}$ and completes the proof.
\end{proof}

 Following the result from Theorem \ref{thm:sampled data zcbf}, the following control is proposed for mechanical systems:
\begin{align} \label{eq:consat sampled data qp}
\begin{split}
\myvar{u}_k^* \hspace{0.1cm} = \hspace{0.1cm} & \underset{\myvar{u} \in \mathbb{R}^m}{\text{argmin}}
\hspace{.1cm} \myvar{u}^T \myvar{u}  - 2 \myvar{u}_{\text{nom}_k}^T \myvar{u}  \\
& \text{s.t.}  \hspace{0.6cm} \apmat{A}(\myvar{x}_k) \myvar{u} \geq \apvar{b}(\myvar{x}_k) + \apscalar{\nu} \myvar{1} \\
& \hspace{0.7cm}  \myvar{u}_{\text{min}} \leq \myvar{u} \leq \myvar{u}_{\text{max}}
\end{split}
\end{align}
where $\myvar{u}_{\text{nom}_k} \in \mathbb{R}^m$ is the sampled nominal control and $\apscalar{\nu} \in \mathbb{R}_{\geq 0}$ is a tuning parameter estimating \eqref{eq:sample data margin} to negate the effects of sampling.

The following theorem guarantees constraint satisfaction via the proposed control \eqref{eq:consat sampled data qp} for mechanical systems:

\begin{thm}\label{thm:consat sampled data qp exact}
Consider the sampled-data system \eqref{eq:sampled data system}. Let $\robset{C}_j$ be defined by  \eqref{eq:robust h}, \eqref{eq:constraint set multiple} for the thrice-continuously differentiable, relative degree two functions $h_j(\myvar{x}): \myset{D}_j \to \mathbb{R}$, defined on the open set $\myset{D}_j \supset \myset{C}_j \supset \robset{C}_j$. Suppose there exists a twice-continuously differentiable, extended class-$\mathcal{K}$ function $\alpha_1$ and locally Lipschitz continuous extended class-$\mathcal{K}$ function $\alpha_2$, such that for $\robscalar{B}_j: \mathscr{E}_j \to \mathbb{R}$ (defined by \eqref{eq:candidate zeroing cbf multiple}) on the open set $\myset{E}_j \supset\robset{B}_j$ (defined by \eqref{eq:set B for zbf}), $\robubset{B}$ and $\robubset{C}$ are compact, and $\robubset{B} \cap \robubset{C}$ is non-empty. Let $\myvar{u}_{nom}$ be a nominal control with $\myvar{u}_{nom_k} = \myvar{u}_{nom}(t = kT)$. Then for a given $T$, there exists $\theta$, $\apscalar{\nu}$, $\beta_j$, $\delta_j \in \mathbb{R}_{\geq 0}$, $j\in [1,l]$ such that for $\myvarnorm{d}_{\infty} \leq \theta$, the control \eqref{eq:consat sampled data qp} applied to \eqref{eq:sampled data system} ensures $\myvar{x}$ remains in $\ubscr{B} \bigcap \ubscr{C}$ for $t \in [0, NT)$.
\end{thm}
\begin{proof}
Existence and uniqueness of $\myvar{u}_k^*$ is ensured by the linear constraints with respect to $\myvar{u}$ in \eqref{eq:consat sampled data qp} \cite{Nocedal2006}. Let $\nu(T)$ be defined from the proof of Theorem \ref{thm:sampled data zcbf}, and choose $\apscalar{\nu}$ sufficiently large such that $\apscalar{\nu} \geq \nu(T)$ for all $j \in [1,l]$, $t \in [0,NT)$. Now, the constraints of \eqref{eq:consat sampled data qp} ensures $\myvar{u}_k^*$ is bounded and $\myvar{u}_k^* \in \ubscr{S}_{u_k}$ such that $\myvar{x}$ remains in $\ubscr{B} \cap \ubscr{C}$ for $t \in [0, NT)$ via Theorem \ref{thm:sampled data zcbf}.
\end{proof}

\begin{remark}
An alternative approach to addressing control barrier functions for sampled-data systems is to employ existing emulation techniques for set stability to the continuous-time controller \eqref{eq:consat qp robust ct} \cite{Grune2008}. However this alternative requires more restrictive conditions including sufficient smoothness of all components of \eqref{eq:consat qp robust ct} in addition to Properties \ref{prop:linear independence} and \ref{prop:complimentary slackness}. On the other hand, by exploiting the properties of sampled-data systems, the proposed method does not require those unnecessary conditions  and is thus applicable to more general systems.
\end{remark}

\section{Grasp Constraint Satisfaction} \label{sec:system model}

In the previous section, a novel control barrier function method was proposed to ensure constraint satisfaction for mechanical systems. In this section, that method is applied to the challenging problem of robotic grasping, in which a robotic hand is required to grasp and manipulate an object to a desired reference pose trajectory. The application of the proposed zeroing control barrier method to this problem is not trivial, and is an additional contribution of this paper. The following sections develop the appropriate model and conventional assumptions required in related grasping work. Then, appropriate analysis is presented to ensure constraint satisfaction for robotic grasping.

\subsection{Hand-Object System}

Consider a fully-actuated, multi-fingered hand grasping a rigid, convex object at $n \in \mathbb{Z}_{>0}$ contact points.  Each finger consists of $m_i \in \mathbb{Z}_{>0}$ revolute joints with smooth, convex fingertips of high stiffness. Let the finger joint configuration be described by the joint angles, $\myvar{q}_i \in \mathbb{R}^{m_i}$. The full hand configuration is defined by the joint angle vector, $\myvar{q} = (\myvar{q}_1, \myvar{q}_2, ..., \myvar{q}_n)^T \in \mathbb{R}^m$, where $m = \sum_{i=1}^n m_i$ is the total number of joints. Let the inertial frame, $\mathcal{P}$, be fixed on the palm of the hand, and a  fingertip frame, $\mathcal{F}_i$, fixed at the point $\myvar{p}_{f_i} \in \mathbb{R}^3$. The translational and rotational velocities of $\mathcal{F}_i$ with respect to $\mathcal{P}$ are denoted $\myvar{v}_{f_i}, \myvar{\omega}_{f_i} \in \mathbb{R}^3$, respectively. The rotation matrix from $\mathcal{F}_i$ to $\mathcal{P}$ is $R_{pf_i} := R_{pf_i}(\myvar{q}_i) \in SO(3)$. The contact frame, $\mathcal{C}_i$, is located at the contact point, $\myvar{p}_{c_i} \in \mathbb{R}^3$.  A visual representation of the contact geometry for the $i$th finger is shown in Figure \ref{fig.contactpic}. 
%
%
%

\begin{figure}[hbtp]
\centering
	\subcaptionbox{ \label{fig.contactpic} }
		{\includegraphics[scale=.37]{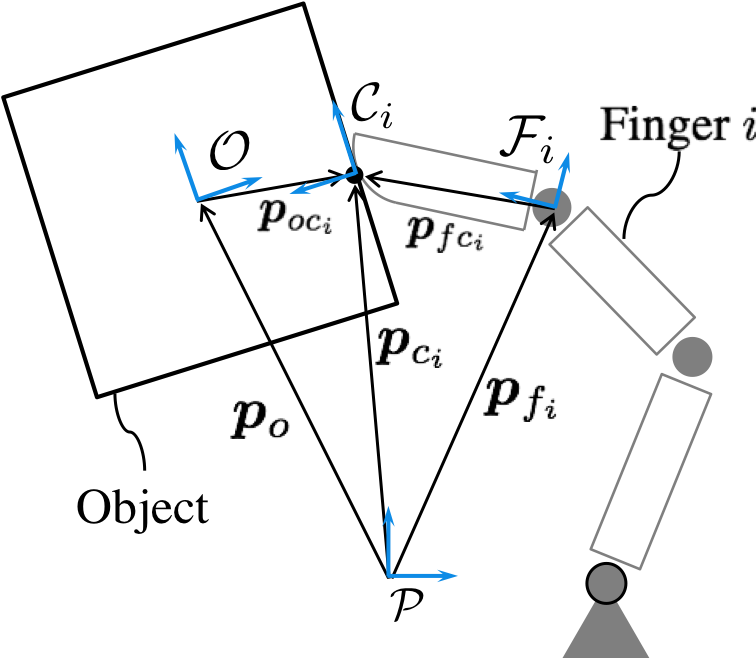}}
	\subcaptionbox{  \label{fig:contact frame pic}}
		{\includegraphics[scale=.29]{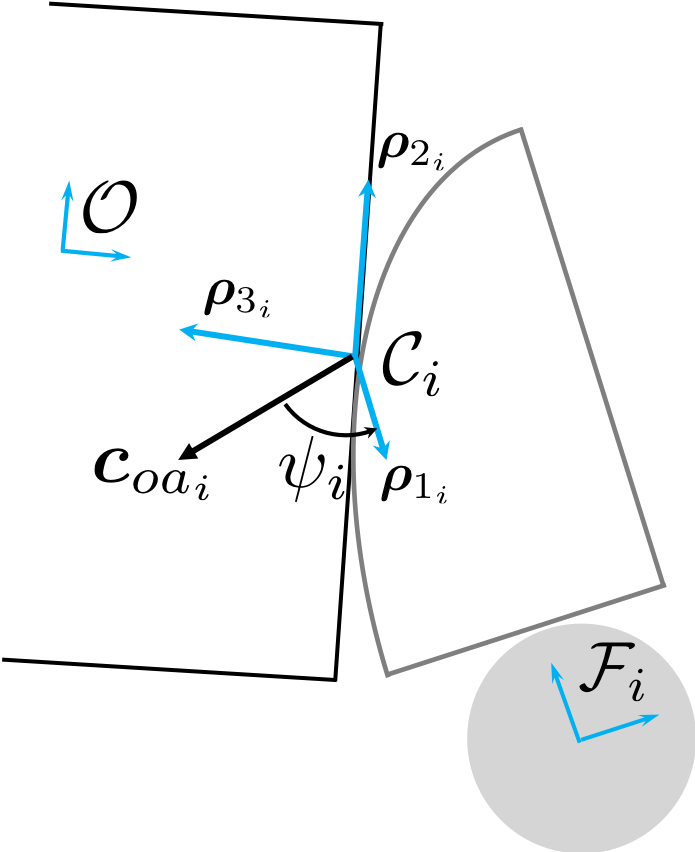}}
	\caption{Hand-object and contact frames for contact $i$.}  \label{fig.contactpic figures}
\end{figure}

The hand dynamics incorporating each finger (i.e. robotic manipulator) of the grasp is as follows \cite{Murray1994}:
\begin{equation} \label{eq:hand dynamics}
M_h \bmdd{q} + C_h \myvardot{q} = -J_h^T \myvar{f}_c +  \myvar{\tau}_e + \myvar{u}
\end{equation}
where $M_{h} := M_{h}(\myvar{q}) \in \mathbb{R}^{m\times m}$ is the inertia matrix, $C_{h} :=  C_{h}(\myvar{q},\myvardot{q}) \in \mathbb{R}^{m \times m}$ is the Coriolis/centrifugal matrix, $J_h := J_h( \myvar{q}, \myvar{p}_{fc}) \in \mathbb{R}^{3n \times m}$ is the hand Jacobian, $\myvar{f}_c \in \mathbb{R}^{3n}$ is the contact force, $\myvar{\tau}_e :=  \myvar{\tau}_e(t) \in \mathbb{R}^m$ is the disturbance torque acting on the joints, and  $\myvar{u} \in \mathbb{R}^m$ is the joint torque control input. Note each $M_h$, $C_h$, $J_h$ is a block diagonal matrix of the individual inertia, Coriolis/centrifugal, and hand Jacobian matrix respectively, and each $\myvar{f}_c$, $\myvar{\tau}_e$, $\myvar{u}$ is a vector concatenation of the individual contact forces, $\myvar{f}_{c_i}$, external torques $ \myvar{\tau}_{e_i}$, and control inputs $ \myvar{u}_i$, respectively. For ease of notation, the matrix $E_i \in \mathbb{R}^{m_i \times m}$ is used to map from the full concatenated vector form to the individual vector, such that for example, $\myvardot{q}_i = E_i \myvardot{q}$. Note the individual hand Jacobian matrix, $J_{h_i}$, is defined by:
\begin{equation}\label{eq:hand jacobian}
J_{h_i}(\myvar{q}_i,\myvar{p}_{fc_i}) = \bracketmat{cc}{I_{3\times 3} & -(\myvar{p}_{fc_i})\times} J_{s_i}(\myvar{q}_i)
\end{equation}
where $J_{s_i}(\myvar{q}_i) \in \mathbb{R}^{6 \times mi}$ is the spatial manipulator Jacobian that maps $\myvardot{q}_i \mapsto (\myvar{v}_{f_i}, \myvar{\omega}_{f_i})$ \cite{Murray1994}. 

Let $\mathcal{O}$ be a reference frame fixed at the object center of mass $\myvar{p}_o \in \mathbb{R}^3$, and $R_{po} \in  SO(3)$ is the rotation matrix, which maps from $\mathcal{O}$ to $\mathcal{P}$. The respective inertial translation and rotational velocities of the object are $\myvar{v}_o, \myvar{\omega}_o \in \mathbb{R}^3$.  The object state is $\myvar{x}_o \in \mathbb{R}^6$, with $\myvardot{x}_o = (\myvar{v}_o, \myvar{\omega}_o)$. The position vector from $\mathcal{O}$ to the respective contact point is $\myvar{p}_{oc_i} \in \mathbb{R}^3$.

The object dynamics are given by \cite{Murray1994}:
\begin{equation} \label{eq:object dynamics}
M_o\bmdd{x}_o + C_o \myvardot{x}_o = G \myvar{f}_c + \myvar{w}_e
\end{equation}
where $M_o := M_o(\myvar{x}_o) \in \mathbb{R}^{6\times 6}$ is the object inertia matrix, $C_o := C_o(\myvar{x}_o,\myvardot{x}_o) \in \mathbb{R}^{6\times 6}$ is the object Coriolis and centrifugal matrices, $G :=  G(\myvar{p}_{oc})  \in \mathbb{R}^{6 \times 3n}$ is the grasp map, and $\myvar{w}_e :=  \myvar{w}_e(t) \in \mathbb{R}^6$ is an external wrench disturbing the object. The grasp map, $G$, maps the contact force, $\myvar{f}_c$, to the net wrench acting on the object and is defined by:
 \begin{equation}\label{eq:grasp map}
G(\myvar{p}_c, \myvar{p}_o) = \left[ \begin{array}{cccc} I_{3\times 3}, & ... ,& I_{3\times 3} ,\\ (\myvar{p}_{c_1}-\myvar{p}_o)\times, & ... ,& (\myvar{p}_{c_n}-\myvar{p}_o)\times \end{array} \right]
\end{equation}
 
When grasping an object, it is important to prevent slip from occurring by ensuring each contact force remains inside the friction cone defined by:
\begin{equation}\label{eq:friction cone}
F_{c_i}= \{ \myvar{f}^{\mathcal{C}_i}_{c_i} \in \mathbb{R}^3 : f_{n_i} \mu \geq \sqrt{f_{x_i}^2 + f_{y_i}^2} \}
\end{equation}
where $\myvar{f}^{\mathcal{C}_i}_{c_i} = (f_{x_i}, f_{y_i}, f_{n_i}) $ is the contact force at $i$ written in frame $\mathcal{C}_i$ with tangential force components $f_{x_i}, f_{y_i} \in \mathbb{R}$ and normal force component $f_{n_i} \in \mathbb{R}$, and $\mu \in \mathbb{R}_{>0}$ is the friction coefficient. The full friction cone is the Cartesian product of all the friction cones: $F_c = F_{c_1} \times ... \times F_{c_n}$. 

 When the contact points do not slip, the following grasp relation holds \cite{Cole1989}:
\begin{equation}\label{eq:grasp constraint}
J_h \myvardot{q} = G^T \myvardot{x}_o
\end{equation}

The following assumptions are made for the grasp:
\begin{assumption}\label{asm:square Jh}
The multi-fingered hand has $m \geq 3 n$ joints.
\end{assumption}
\begin{assumption}\label{asm:full rank G}
The given multi-fingered grasp is such that $G$ is full rank and $\mathcal{R}(G) \bigcap \text{Int}(FC) \neq \emptyset$.
\end{assumption}
\begin{assumption}\label{asm:smooth surfaces}
The system dynamics \eqref{eq:hand dynamics}, \eqref{eq:object dynamics}, and local fingertip/object contact surfaces are smooth.
\end{assumption}
%

\begin{remark} \label{rm:Jh to slip}
 Assumption \ref{asm:full rank G} ensures the grasp is force-closure, which means that for any given object wrench, there exists a contact force that produces the given wrench and also lies inside the friction cone \cite{Cole1989}. This force-closure condition is a conventional assumption in robotic grasping literature, and can be ensured by a high-level grasp planner \cite{Hang2016}. 
\end{remark}

\subsection{Hand-Contact Kinematics}\label{ssec:geometric parameters}

Here the differential geometric modeling of rolling contacts is reviewed as presented in \cite{Murray1994, Montana1988}. Note, the subscript $co$ will refer to the object surface of the contact, and the subscript $cf$ refers to the fingertip surface of the contact. At each contact point, the contact surfaces are parameterized by local coordinates  $\myvar{\xi}_{co_i} = (a_{co_i}, b_{co_i}), \myvar{\xi}_{cf_i} = (a_{cf_i}, b_{cf_i})$. The relation between the local coordinates and contact position vectors are defined by smooth mappings: $\myvar{p}^{\mathcal{F}_i}_{fc_i} = \myvar{c}_{cf_i}(\myvar{\xi}_{cf_i}), \myvar{p}^{\mathcal{O}}_{oc_i} = \myvar{c}_{co_i}(\myvar{\xi}_{co_i})$. 

The geometric parameters including the metric tensor, curvature tensor, and torsion tensor are used to define the rolling contact kinematics. For ease of notation, $\myvar{c}_{fa}, \myvar{c}_{fb}$ respectively denote $ \frac{\partial \myvar{c}_{cf_i}}{\partial a_{cf_i}}$ and  $\frac{\partial \myvar{c}_{cf_i}}{\partial b_{cf_i}}$. Similarly let $\myvar{c}_{oa}, \myvar{c}_{ob}$ respectively denote $ \frac{\partial \myvar{c}_{co_i}}{\partial a_{co_i}}$ and  $\frac{\partial \myvar{c}_{co_i}}{\partial b_{co_i}}$.

The Gauss frame is used to define the contact frame $\mathcal{C}_i$, which is depicted in Figure \ref{fig:contact frame pic}:
\begin{equation}\label{eq:gauss frame}
R_{fc_i} = \bracketmat{ccc}{\myvar{\rho}_1 & \myvar{\rho}_2 & \myvar{\rho}_3} = \bracketmat{ccc}{  \frac{\myvar{c}_{fa}}{|| \myvar{c}_{fa} ||} & \frac{\myvar{c}_{fa}}{|| \myvar{c}_{fb} ||} & \frac{\myvar{c}_{fa} \times \myvar{c}_{fb}}{ || \myvar{c}_{fa} \times \myvar{c}_{fb} ||} }
\end{equation}

where $R_{fc_i} \in SO(3)$ maps $\mathcal{C}_i$ to $\mathcal{F}_i$.

The metric tensor, $M_{cf_i} := M_{cf_i}(\myvar{\xi}_{cf_i}) \in \mathbb{R}^{2\times 2}$, curvature tensor, $K_{cf_i} := K_{cf_i}(\myvar{\xi}_{cf_i}) \in \mathbb{R}^{2\times 2}$, and torsion tensor, $T_{cf_i} := T_{cf_i}(\myvar{\xi}_{cf_i})  \in \mathbb{R}^{2\times 1}$ are defined by:
\begin{equation}\label{eq:kinematic tensors}
\begin{split}
M_{cf_i} &= \bracketmat{cc}{ || \myvar{c}_{fa} || & 0 \\ 0 & ||\myvar{c}_{fb} || } \\
 K _{cf_i}&= \bracketmat{c}{\myvar{\rho}_1^T \\ \myvar{\rho}_2^T} \bracketmat{cc}{ \frac{\partial \myvar{\rho}_3/\partial a_{cf_i}}{|| \myvar{c}_{fa} ||} & \frac{\partial \myvar{\rho}_3/\partial b_{cf_i}}{|| \myvar{c}_{fb} ||}}\\
 T_{cf_i} &= \myvar{\rho}_2^T \bracketmat{cc}{  \frac{\partial \myvar{\rho}_1 / \partial a_{cf_i}}{|| \myvar{c}_{fa} ||}  & \frac{\partial \myvar{\rho}_1/\partial b_{cf_i}}{|| \myvar{c}_{fb} ||}}
\end{split}
\end{equation}

The same geometric parameters for the object, $M_{co_i} := M_{co_i}(\myvar{\xi}_{co_i}) \in \mathbb{R}^{2\times 2}, K_{co_i} := K_{co_i}(\myvar{\xi}_{co_i}) \in \mathbb{R}^{2\times 2}, T_{co_i} = T_{co_i}(\myvar{\xi}_{co_i}) \in \mathbb{R}^{2\times 1}$, are defined by appropriate substitution of $\myvar{\xi}_{cf_i}$ with $\myvar{\xi}_{co_i}$ in \eqref{eq:gauss frame}-\eqref{eq:kinematic tensors}: 
\begin{equation}\label{eq:obj kinematic tensors}
\begin{split}
M_{co_i} &= \bracketmat{cc}{ || \myvar{c}_{oa}|| & 0 \\ 0 & ||\myvar{c}_{ob} || } \\
K _{co_i} &= \bracketmat{c}{\myvar{\rho}_1^T \\ \myvar{\rho}_2^T} \bracketmat{cc}{ \frac{\partial \myvar{\rho}_3/\partial a_{co_i}}{|| \myvar{c}_{oa} ||} & \frac{\partial \myvar{\rho}_3/\partial b_{co_i}}{|| \myvar{c}_{ob} ||}} \\
T_{co_i} &= \myvar{\rho}_2^T \bracketmat{cc}{  \frac{\partial \myvar{\rho}_1 / \partial a_{co_i}}{|| \myvar{c}_{oa} ||}  & \frac{\partial \myvar{\rho}_1/\partial b_{co_i}}{|| \myvar{c}_{ob} ||}}
\end{split}
\end{equation}

Now the equations of motion for $\myvar{\xi}_{cf_i}$ and $\myvar{\xi}_{co_i}$ are defined as follows:
\begin{multline}\label{eq:contact kinematics}
\myvardot{\xi}_{cf_i} = M_{cf_i}^{-1} ( K_{cf_i} \\ + R_{\psi_i} K_{co_i} R_{\psi_i} )^{-1} \bracketmat{ccc}{0 & -1 & 0 \\ 1 & 0 & 0}
 R_{c_i p} (\myvar{\omega}_{f_i} - \myvar{\omega}_o)
\end{multline}
\begin{multline}\label{eq:contact object kinematics}
\myvardot{\xi}_{co_i} = M_{co_i}^{-1} R_{\psi_i}(K_{cf_i} \\ + R_{\psi_i}K_{co_i} R_{\psi_i})^{-1}  \bracketmat{ccc}{0 & -1 & 0 \\ 1 & 0 & 0}
R_{c_i p} (\myvar{\omega}_{f_i} - \myvar{\omega}_o) 
\end{multline}
where 
\begin{equation}\label{eq:R_psi}
R_{\psi_i} = \bracketmat{cc}{ \cos(\psi_i) & -\sin(\psi_i) \\ -\sin(\psi_i) & -\cos(\psi_i)},
\end{equation}
and $R_{c_i p} = R_{fc_i}^T R_{pf_i}^T $ maps $\mathcal{P}$ to $\mathcal{C}_i$. The contact angle dynamics is defined by:
\begin{equation}\label{eq:contact psi dynamics}
\dot{\psi}_i = T_{cf_i} M_{cf_i} \myvardot{\xi}_{cf_i} + T_{co_i} M_{co_i} \myvardot{\xi}_{co_i}
\end{equation}
where  $\psi_i \in \mathbb{R}$ is the angle between $\frac{\partial \myvar{c}_{co_i} }{\partial a_{co_i}}$ and  $\frac{\partial \myvar{c}_{cf_i} }{\partial a_{cf_i}}$ (see Figure \ref{fig:contact frame pic}).

It is important to note the chosen parameterizations must satisfy the following assumption for \eqref{eq:contact kinematics}, \eqref{eq:contact object kinematics}  to be well-defined \cite{Murray1994}:
\begin{assumption}\label{asm:geometric parameterization}
The parameterizations are orthogonal such that $\frac{\partial \myvar{c}_{f_i} }{\partial a_{cf_i}}^T \frac{\partial \myvar{c}_{f_i} }{\partial b_{f_i}} = 0$, $\frac{\partial \myvar{c}_{o_i} }{\partial a_{o_i}}^T \frac{\partial \myvar{c}_{o_i} }{\partial b_{o_i}} = 0$, and $M_{cf_i}, K_{cf_i}, T_{cf_i}, M_{co_i}, K_{co_i}, T_{co_i} $ are defined for all $\myvar{\xi}_{cf_i}$ on the fingertip surface, and $\myvar{\xi}_{co_i}$ on the object surface, respectively.
\end{assumption}

\subsection{Problem Formulation}\label{ssec:problem formulation}

The formal definition of the grasp constraint satisfaction problem is presented here.

\subsubsection*{Contact Force Set}

No slip is ensured by guaranteeing that the contact forces remain inside the friction cone, $\mathcal{FC}$. A well-known technique is to approximate the friction cone by a pyramid, which results in the following linear constraint condition \cite{Kerr1986}:
\begin{equation}\label{eq:no slip condition}
\Lambda(\mu) R_{cp} \myvar{f}_c > 0
\end{equation}
where $\Lambda(\mu) \in \mathbb{R}^{l_s n \times 3 n}$ defines a pyramid of $l_s \in \mathbb{Z}_{>0}$ faces used to a approximate the friction cone \cite{Kerr1986}, and $ R_{cp} \in \mathbb{R}^{3n \times 3n} $ is the block diagonal matrix of all $R_{c_i p}$ for $i \in [1,n]$. Let the set of constraint admissible contact forces be defined as:
\begin{equation}\label{eq:safe contact forces}
\ubscr{C}_{f} = \{ \myvar{f}_c \in \mathbb{R}^{3n}: \Lambda(\mu) R_{cp} \myvar{f}_c > 0 \}
\end{equation}

\subsubsection*{Joint Angle Set}
Over-extension of the joints is prevented by constraining the joint angles within a feasible workspace. The constraints on the joint angles are defined by the following box constraints:
\begin{equation}\label{eq:joint angle limits}
 \begin{split}
  h_{q\text{max}_j}(\myvar{q}) &= - \myvar{i}_j \myvar{q} + q_{\text{max}_j},  \forall j\in [1,m] \\
  h_{q\text{min}_j}(\myvar{q}) &=  \myvar{i}_j \myvar{q} - q_{\text{min}_j}, \forall j\in [1,m]
 \end{split}
\end{equation} 
where $\myvar{i}_j \in \mathbb{R}^{1\times m}$ is the $j$th row of $I_{m\times m}$ and $q_{\text{max}_j}, q_{\text{min}_j} \in \mathbb{R}_{\geq0}$ define the joint angle limits, which omit singular hand configurations. The set of constraint admissible joint angles is defined by:
\begin{equation}\label{eq:safe joint space}
\ubscr{C}_{q} = \{ \myvar{q} \in \mathbb{R}^m: \forall j \in [1,m]: h_{q\text{max}_j}(\myvar{q})  \geq 0,  h_{q\text{min}_j} (\myvar{q}) \geq 0 \} 
\end{equation}

\subsubsection*{Contact Location Set}
Excessive rolling occurs when the contact locations exceed the workspace of the fingertip. Many existing tactile sensors are designed as fingertips with flat, hemispherical, or other relatively simple geometric surface \cite{Kappassov2015}, which can be appropriately modeled with geometric parameterizations \cite{Murray1994}. The benefit of the geometric modeling is not only that it can be applied to these fingertip shapes, but the fingertip workspace can be defined as box constraints: 
\begin{equation}\label{eq:fingertip workspace constraints}
\begin{split}
h_1(\myvar{\xi}_{cf_i}) & = a _{cf_i} - a_{\text{min}},\  h_2(\myvar{\xi}_{cf_i}) =  - a _{cf_i} + a_{\text{max}}  \\ 
h_3(\myvar{\xi}_{cf_i}) & =   b _{cf_i} - b_{\text{min}},\  h_4(\myvar{\xi}_{cf_i}) =  - b _{cf_i} + b_{\text{max}}  
\end{split}
\end{equation}
where $a_{\text{min}}, a_{\text{max}}, b_{\text{min}}, b_{\text{max}} \in \mathbb{R}$ define the boundary of the fingertip surface. Each $h_{r_j}$ defines the box constraints such that if $h_{r_j} \geq 0, \forall j \in [1,4]$, then the contact point is in the fingertip workspace. The set of allowable contact locations for each contact is:
\begin{equation}\label{eq:safe contact locations}
\ubscr{C}_{r_i} = \{ \myvar{\xi}_{cf_i} \in \mathbb{R}^2: \forall j \in [1,4]: h_{r_j}(\myvar{\xi}_{cf_i})  \geq 0 \}
\end{equation}
and the full set of feasible contact locations is:
\begin{equation}\label{eq:safe contact locations multiple}
\ubscr{C}_r = \{  \myvar{\xi}_{cf} \in \mathbb{R}^{2n}:  \forall i \in [1,n]: \myvar{\xi}_{cf_i} \in\mathscr{C}_{r_i} \}
\end{equation}

\subsubsection*{Constraint-admissible States}
Let $\mathscr{H}$, defined by:
\begin{equation}\label{eq:set H}
 \mathscr{H} =  \ubscr{C}_{f} \times \ubscr{C}_q \times \ubscr{C}_r
\end{equation}
denote the set of grasp constraint admissible states. In the set $\mathscr{H}$, the hand configuration is non-singular, and, by Assumption \ref{asm:square Jh}, $J_h$ is full rank with rank $3n$. Furthermore, the contact points do not slip in $\mathscr{H}$ and so the grasp relation \eqref{eq:grasp constraint} holds. These properties of $\mathscr{H}$ are attractive, and exploited in related literature where the states are assumed to remain in $\mathscr{H}$ without any guarantee of such a claim \cite{Ozawa2017}. 

The problem to be addressed is formally stated as:

\begin{problem}\label{pr:grasp constraint satisfaction}
Suppose the hand-object system satisfies Assumptions \ref{asm:square Jh}-\ref{asm:geometric parameterization}, and consider the set of constraint-admissible states $\mathscr{H}$ defined by \eqref{eq:set H}. Determine a control law that ensures forward invariance of $\mathscr{H}$.
\end{problem}

\subsection{Proposed Grasp Constraint Satisfying Control}\label{sec:proposed solution}

In the following section, the control input $\myvar{u}_k$ is proposed to ensure no slip, no over-extension, and no excessive rolling are satisfied. Satisfaction of no over-extension and no excessive rolling constraints is addressed using the zeroing control barrier functions from Section \ref{sec:control barrier fcns}, while satisfaction of the no slip condition is developed as an extension of \cite{ShawCortez2018b}. The input constraints are then combined to define the proposed controller to address Problem \ref{pr:grasp constraint satisfaction}.

\subsubsection{Grasp Constraint Conditions}

In practice, the robotic hand may have limited information about the object and/or may only have access to embedded sensors including joint angle senors and tactile sensors to gather information of the grasp. Here model uncertainties are taken into account to develop a controller that ensures forward invariance of the grasp constraint-admissible set $\mathscr{H}$. For notation, an approximation is denoted by $\hat{(\cdot)}$ and the error by $\Delta(\cdot)$. For example, the object inertia matrix is defined as $M_o = \apmat{M}_o + \Delta(M_o)$.

To develop a robust grasp constraint satisfying controller, the approach taken here is to exploit the robustness properties of zeroing control barrier functions presented in Section \ref{sec:control barrier fcns}. To do so, first the contact force, $\myvar{f}_c$, must be analyzed with respect to model uncertainties as it defines the interaction between the hand and object. Differentiation of \eqref{eq:grasp constraint}, and substitution of \eqref{eq:hand dynamics} and \eqref{eq:object dynamics}  provides an expression for the contact forces:
\begin{multline}\label{eq:contact force}
\myvar{f}_c = B_{ho}^{-1}  \Big( J_h M_h^{-1} ( -C_h \myvardot{q} + \myvar{u}_k + \myvar{\tau}_e) + \dot{J}_h \myvardot{q} 
 -\dot{G}^T \myvardot{x}_o \\ + G^T M_o^{-1}( C_o \myvardot{x}_o - \myvar{w}_e) \Big)
\end{multline}
where $B_{ho} = ( J_h M_h^{-1} J_h^T + G^T M_o^{-1} G )$. Note that by Assumptions \ref{asm:square Jh} and \ref{asm:full rank G}, $B_{ho}^{-1}$ is well defined for all states in $\mathscr{H}$. Furthermore, as with previous notation, $\apvar{f}_c$ will denote the approximate contact force, and $\Delta(\myvar{f}_c)$ denotes the error between the nominal and exact contact force.

From \eqref{eq:contact force}, it is clear that the contact force is affected by model uncertainties, which must be appropriately taken into consideration. In addressing robustness, note that complete lack of information of the system is extremely limiting, and a hard problem to address. For practical consideration, several simplifying assumptions are made, which are listed as follows:

\begin{assumption}\label{asm:bounded friction}
The approximate friction coefficient satisfies: $\apscalar{\mu} \leq \mu(T) \leq \mu$, where $\mu(T)$ is the conservative friction coefficient to address inter-sampling behavior as discussed in \cite{ShawCortez2018b}.
\end{assumption}

\begin{assumption}\label{asm:bounded approximations}
All approximation errors, including $\Delta(\myvar{f}_c)$, are bounded.
\end{assumption}

\begin{remark}
From Assumption \ref{asm:bounded friction}, $\apscalar{\mu}$ defines the lower bound on the allowable contact friction. In practice, the designer should choose $\apscalar{\mu}$ for the specified task where a small $\apscalar{\mu}$ could be chosen to grasp a larger range of objects, including slippery objects such as ice, however more  control effort will be required to do so. Assumption \ref{asm:bounded approximations} is valid as contact forces never tend to infinity in practice.
\end{remark}

In the following sections, constraints are defined on $\myvar{u}_k$ to ensure $\mathscr{H}$ is forward invariant despite perturbations that result from model uncertainties.

\subsubsection*{Contact Force Constraint}

Starting with the no slip condition, first note a property of the friction cone in relation to $\mu$ in \eqref{eq:no slip condition}. The use of the more conservative friction coefficient, $\apscalar{\mu}$, that satisfies Assumption \ref{asm:bounded friction} in \eqref{eq:no slip condition} implies that \eqref{eq:no slip condition} holds for the true friction coefficient, $\mu$. Thus $\apscalar{\mu}$ can be directly substituted for $\mu$ in \eqref{eq:no slip condition} to ensure no slip. 

Next, the model uncertainties are explicitly investigated to develop a robust means of ensuring no slip. To do so, let the approximate contact force be explicitly defined as:
\begin{multline}\label{eq:contact force approximation}
\apvar{f}_c = \apmat{B}_{ho}^{-1}  \Big( \apmat{J}_h \apmat{M}_h^{-1} ( -\apmat{C}_h \myvardot{q} + \myvar{u}_k + \apvar{\tau}_e) + \apmatdot{J}_h \myvardot{q} \\
-\apmatdot{G}^T \apvardot{x}_o + \apmat{G}^T \apmat{M}_o^{-1}( \apmat{C}_o  \apvardot{x}_o  - \apvar{w}_e) \Big)
\end{multline}
where  $\apmat{B}_{ho} = \apmat{J}_h \apmat{M}_h^{-1} \apmat{J}_h^T + \apmat{G}^T \apmat{M}_o^{-1} \apmat{G}$.

By substituting $\myvar{f}_c$ with $\apvar{f}_c + \Delta(\myvar{f}_c)$, where $\apvar{f}_c$ is defined as in \eqref{eq:contact force approximation}, and $\apscalar{\mu}$ for $\mu$ in \eqref{eq:no slip condition}, the following relation must hold to ensure no slip:
\begin{multline}\label{eq:no slip condition with Delta fc}
\Lambda(\apscalar{\mu}) R_{cp} \apmat{B}_{ho}^{-1} \apmat{J}_h \apmat{M}_h^{-1} \myvar{u}_k   > \Lambda(\apscalar{\mu}) R_{cp} \Big( \apmat{B}_{ho}^{-1} \apmat{J}_h \apmat{M}_h^{-1} (\apmat{C}_h \myvardot{q} \\- \apvar{\tau}_e)
- \apmatdot{J}_h \myvardot{q} + \apmatdot{G}^T \apvardot{x}_o - \apmat{G}^T \apmat{M}_o^{-1}( \apmat{C}_o \apvardot{x}_o - \apvar{w}_e) \Big)  -\Lambda(\apscalar{\mu}) R_{cp} \Delta(\myvar{f}_c)
\end{multline}

From Assumption \eqref{asm:bounded approximations} it follows that the term $\Lambda(\apscalar{\mu})R_{cp} \Delta(\myvar{f}_c)$ is bounded. Thus to ensure satisfaction of \eqref{eq:no slip condition with Delta fc}, and thus \eqref{eq:no slip condition}, a tuning parameter $\varepsilon \in \mathbb{R}_{\geq 0}$ is chosen to be larger than the bound on $\Lambda(\apscalar{\mu})R_{cp} \Delta(\myvar{f}_c)$, which is incorporated in the following no slip constraint:
\begin{multline}\label{eq:no slip constraint}
\Lambda(\apscalar{\mu}) R_{cp} \apmat{B}_{ho}^{-1} \apmat{J}_h \apmat{M}_h^{-1} \myvar{u}_k   \geq \Lambda(\apscalar{\mu}) R_{cp} \Big( \apmat{B}_{ho}^{-1} \apmat{J}_h \apmat{M}_h^{-1} (\apmat{C}_h \myvardot{q} \\ - \apvar{\tau}_e) 
- \apmatdot{J}_h \myvardot{q} + \apmatdot{G}^T \apvardot{x}_o - \apmat{G}^T \apmat{M}_o^{-1}( \apmat{C}_o \apvardot{x}_o - \apvar{w}_e) \Big) + \myvar{1}\varepsilon
\end{multline}
The satisfaction of \eqref{eq:no slip constraint} by $\myvar{u}$ thus ensures that the no slip condition \eqref{eq:no slip condition} is satisfied for an appropriately chosen $\varepsilon$. Let the set of admissible control torques for $\mathscr{C}_f$ be $\mathscr{S}_{u_f} = \{ \myvar{u} \in \mathbb{R}^m: \eqref{eq:no slip constraint} \ \text{holds} \}$. This result is summarized in the following lemma:

\begin{lemma}\label{lemma:contact force invariance}
Under Assumptions \ref{asm:square Jh}-\ref{asm:bounded approximations}, there exists a $\varepsilon^* \in \mathbb{R}_{\geq 0}$ such that if $\myvar{u}_k$ satisfies \eqref{eq:no slip constraint} for $\varepsilon > \varepsilon^*$, then $\ubscr{C}_f$ is forward invariant.
\end{lemma}
\begin{remark}
As discussed in \cite{ShawCortez2018b}, slip can occur if perturbations sufficiently increase the ratio of tangential to normal forces with respect to the contact surface. The use of $\varepsilon$ in \eqref{eq:no slip constraint} enforces a lower bound of $\varepsilon / \apscalar{\mu}$ on the normal force. Essentially this ensures the hand is squeezing the object hard enough to resist such disturbances. Also, the use of $\varepsilon$ also ensures robustness to sampling time effects \cite{ShawCortez2018b}.
\end{remark}

\subsubsection*{Joint Angle Constraint}

The zeroing control barrier functions from Section \ref{sec:control barrier fcns} are used here to guarantee that the hand joints remain inside a desired joint space to prevent over-extension. First, robustness margins are incorporated into the functions $h_q$, and a conservative $\ubscr{C}_q$ is defined:
\begin{equation}\label{eq:joint angle limits robust}
 \begin{split}
 \robscalar{h}_{q\text{max}_j}(\myvar{q}) &= h_{q\text{max}_j}(\myvar{q}) - \delta_{q\text{max}_j}, \forall j\in [1,m] \\
 \robscalar{h}_{q\text{min}_j}(\myvar{q}) &=  h_{q\text{min}_j}(\myvar{q}) -\delta_{q\text{min}_j}, \forall j\in [1,m]
 \end{split}
\end{equation} 
\begin{equation}\label{eq:Cq set robust}
\hat{\ubscr{C}}_q = \{ \myvar{q} \in \mathbb{R}^m: \forall j \in [1,m]: \apscalar{h}_{q\text{max}_j} \geq 0, \apscalar{h}_{q\text{min}_j} \geq 0 \}
\end{equation}
where $\delta_{q\text{max}_j}, \delta_{q\text{max}_j} \in \mathbb{R}_{\geq 0}$ are the robustness margins. 

Similarly, the following zeroing control barrier functions are defined with robustness margins to prevent over-extension:
\begin{equation} \label{eq:joint angle zcbfs}
\begin{split}
\robscalar{B}_{q\text{max}_j}(\myvar{q}, \myvardot{q}) &= \hdot{h}_{q\text{max}_j} + \alpha_1(\hat{h}_{q\text{max}_j}) - \beta_{q\text{max}_j}, \ j \in [1,m] \\
\robscalar{B}_{q\text{min}_j}(\myvar{q}, \myvardot{q}) &= \hdot{h}_{q\text{min}_j} + \alpha_1(\hat{h}_{q\text{min}_j}) - \beta_{q\text{min}_j}, \  j \in [1,m] 
\end{split}
\end{equation}
where $\alpha_1(h)$ is a twice continuously differentiable, extended class-$\mathcal{K}$ function, and $\beta_{q\text{max}_j}, \beta_{q\text{min}_j} \in \mathbb{R}_{\geq 0}$ define the robustness margins. Let $\robubset{B}_q$ be defined by:
\begin{equation}\label{eq:B set joint angles}
\robubset{B}_q = \{ (\myvar{q}, \myvardot{q}) \in \mathbb{R}^{2m}: \forall j \in [1,m]:  \robscalar{B}_{q\text{max}_j} \geq 0,\\
 \robscalar{B}_{q\text{min}_j} \geq 0  \}
\end{equation}

The zeroing control barrier functions are applied to the dynamics of $\myvar{q}$ where the control input appears, namely $\bmdd{q}$. However \eqref{eq:hand dynamics} does not fully represent the effect of $\myvar{u}$ as the contact forces, $\myvar{f}_c$, are also dependent on $\myvar{u}_k$ as shown in \eqref{eq:contact force}. Substitution of \eqref{eq:contact force} in \eqref{eq:hand dynamics} define the proper system dynamics for the constraint set $\ubscr{C}_q$, which is omitted here for brevity.

Following Theorem \ref{thm:sampled data zcbf}, the following constraint is defined to ensure forward invariance of $\ubscr{C}_q$:
\begin{equation}\label{eq:joint position constraint}
\apmat{A}_q \myvar{u}_k \geq \apvar{b}_q
\end{equation} 
where $\apmat{A}_q \in \mathbb{R}^{2m \times m}$ and $ \apvar{b}_q \in \mathbb{R}^{2m}$ are defined in the Appendix. 

Let the set of admissible control torques for $\robubset{C}_q$ be $\robubset{S}_{u_q} = \{ \myvar{u}_k \in \mathbb{R}^m: k \in [1,N]: \apmat{A}_q \myvar{u}_k \geq \apvar{b}_q + \apscalar{\nu}_q \myvar{1}\}$, for the sampling time margin $\apscalar{\nu}_q \in \mathbb{R}_{\geq0}$.

\subsubsection*{Contact Location Constraint}

The zeroing control barrier functions are also used to ensure the contact points remain in the fingertip workspace. Robustness margins are incorporated into $h_{r_l}$ and used to define the conservative $\mathscr{C}_r$:
\begin{equation}\label{eq:jcontact location limits robust}
 \begin{split}
\hat{h}_{r_1}(\myvar{\xi}_{cf_i}) & = h_{r_1}(\myvar{\xi}_{cf_i}) - \delta_{r_1}, \ \hat{h}_{r_2}(\myvar{\xi}_{cf_i}) =  h_{r_2}(\myvar{\xi}_{cf_i}) - \delta_{r_2} \\ 
\hat{h}_{r_3}(\myvar{\xi}_{cf_i}) & =   h_{r_3}(\myvar{\xi}_{cf_i})  - \delta_{r_3}, \ \hat{h}_{r_4}(\myvar{\xi}_{cf_i}) =  h_{r_4}(\myvar{\xi}_{cf_i})   - \delta_{r_4}
 \end{split}
\end{equation} 
\begin{equation}\label{eq:Cr set robust}
\hat{\ubscr{C}}_r = \{ \myvar{\xi}_{cf} \in \mathbb{R}^m: \forall l \in [1,4]: \forall i \in [1,n]: \apscalar{h}_{r_l}(\myvar{\xi}_{cf_i}) \geq 0 \}
\end{equation}
where $\delta_{r_l} \in \mathbb{R}_{\geq0}$ define the robustness margins for $l \in [1,4]$.

Let the robust zeroing control barrier functions to prevent excessive rolling be defined by:
\begin{equation}\label{eq:control barrier function}
\hat{B}_{r_l} ( \myvar{\xi}_{cf_i}, \myvardot{\xi}_{cf_i})  = \hdot{h}_{r_l}(\myvar{\xi}_{cf_i}) + \alpha_1(\hat{h}_{r_l}) - \beta_{r_l},  \  l \in [1,4]
\end{equation}
where $\beta_{r_l} \in \mathbb{R}_{\geq0}$ defines the robustness margins for $l \in [1,4]$. Let $\hat{\ubscr{B}_r}$ be defined by:
\begin{multline}
\hat{\ubscr{B}}_r = \{ ( \myvar{\xi}_{cf}, \myvardot{\xi}_{cf}) \in \mathbb{R}^{4n}: \forall l \in [1,4]: \forall i \in [1,n]: \\
 \hat{B}_{r_l}( \myvar{\xi}_{cf_i}, \myvardot{\xi}_{cf_i}) \geq 0\}
\end{multline}

The zeroing control barrier functions are applied to the dynamics of $\myvar{\xi}_{cf_i}$ where the control input appears, namely $\bmdd{\xi}_{cf_i}$. The derivation of $\bmdd{\xi}_{cf_i}$ is quite involved and is broken down into the following steps. For ease of notation let $H_i := H(\myvar{\xi}_{cf_i}, \myvar{\xi}_{co_i}, \psi_i)$ be defined by:
\begin{equation}
H_i = M_{cf_i}^{-1} ( K_{cf_i} + R_{\psi_i} K_{co_i} R_{\psi_i} )^{-1}  \bracketmat{ccc}{0 & -1 & 0 \\ 1 & 0 & 0}
\end{equation} 
First, $J_{s_i} E_i \myvardot{q}$ is substituted for $\myvar{\omega}_{f_i}$ in \eqref{eq:contact kinematics}, which is then differentiated, resulting in :
\begin{multline}\label{eq:contact kinematic double derivative1}
\bmdd{\xi}_{cf_i} = \Big( \dot{H_i} R_{c_i p} + H_i \dot{R}_{c_i p} \Big) (\myvar{\omega}_{f_i} - \myvar{\omega}_o)\\
+ H_i R_{c_i p}\bracketmat{cc}{0_{3\times 3} & I_{3\times 3}} \Big( \dot{J}_{s_i} \myvardot{q}_i + J_{s_i} E_i \bmdd{q}  - \bmdd{x}_o    \Big)
\end{multline}
Further substitution of \eqref{eq:hand dynamics} and \eqref{eq:object dynamics} into \eqref{eq:contact kinematic double derivative1} results in:
\begin{multline}\label{eq:contact kinematic double derivative2}
\bmdd{\xi}_{cf_i} = \Big( \dot{H}_i R_{c_i p} + H_i \dot{R}_{c_i p} \Big) (\myvar{\omega}_{f_i} - \myvar{\omega}_o) \\
+ H_i R_{c_i p}\bracketmat{cc}{0_{3\times 3} & I_{3\times 3}}  \Big( \dot{J}_{s_i} \myvardot{q}_i 
+ J_{s_i} E_i M_h^{-1}( -C_h \myvardot{q} \\ 
- J_h^T \myvar{f}_c + \myvar{\tau}_e + \myvar{u}_k  - M_o^{-1} ( -C_o \myvardot{x}_o + G \myvar{f}_c + \myvar{w}_e)    \Big)
\end{multline}
Finally, the contact force \eqref{eq:contact force} is then substituted into \eqref{eq:contact kinematic double derivative2}, which is omitted here for brevity. Following the approach from Theorem \ref{thm:sampled data zcbf}, the following condition must be satisfied to ensure forward invariance of $\ubscr{C}_r$:
\begin{equation}\label{eq:Linear u constraint r}
\apmat{A}_r \myvar{u}_k \geq \apvar{b}_r 
\end{equation}
where $\apmat{A}_r \in \mathbb{R}^{4n \times m}$, $\apvar{b}_r \in \mathbb{R}^{4n}$ are defined in the Appendix. Let the set of admissible control torques for ensuring the contacts remain inside the fingertip workspace be $\robubset{S}_{u_r} = \{ \myvar{u}_k \in \mathbb{R}^m :k \in [0, N]: \apmat{A}_r \myvar{u}_k \geq \apvar{b}_r + \apscalar{\nu}_r \myvar{1} \}$, for the sampling time margin $\apscalar{\nu}_r \in \mathbb{R}_{\geq0}$.

\subsubsection*{Actuator Constraints}

Finally, actuators are limited to a finite actuation range in real-life applications. To ensure the proposed controller is implementable on real systems the following actuator constraint is defined:
\begin{equation}\label{eq:actuator constraint}
\myvar{u}_{\text{min}} \leq \myvar{u}_k\leq  \myvar{u}_{\text{max}}
\end{equation}
where $\myvar{u}_{\text{min}}, \myvar{u}_{\text{max}} \in \mathbb{R}^m$ denote the minimum and maximum allowable torque values, respectively. Let the set of bounded control torques be $\ubscr{S}_{u_\tau} = \{ \myvar{u}_k \in \mathbb{R}^m: k \in [0, N]: \myvar{u}_{\text{min}} \leq \myvar{u}_k \leq  \myvar{u}_{\text{max}} \}$.

\subsubsection*{Proposed Control}

Thus far, the constraints  \eqref{eq:no slip constraint},  \eqref{eq:joint position constraint}, and \eqref{eq:Linear u constraint r} have been presented to prevent slip, joint over-extension, and excessive rolling, respectively. The proposed control admits a nominal manipulation controller, $\myvar{u}_{\text{nom}} \in \mathbb{R}^m$, and outputs a control torque that minimizes $||\myvar{u} - \myvar{u}_{\text{nom}} ||^2$, while adhering to the grasp constraints \eqref{eq:no slip constraint}, \eqref{eq:joint position constraint}, \eqref{eq:Linear u constraint r}, and \eqref{eq:actuator constraint}. The proposed control is:
\begin{align} \label{eq:safe control qp sampled data}
\begin{split}
\myvar{u}_k^* \hspace{0.1cm} = \hspace{0.1cm} & \underset{\myvar{u}}{\text{argmin}}
\hspace{.1cm} \myvar{u}^T \myvar{u}  - 2 \myvar{u}_{\text{nom}_k}^T \myvar{u}  \\
& \text{s.t.} 
\hspace{0.6cm} \apmat{A}_k \myvar{u} \geq \apvar{b}_k + \apscalar{\nu}_h \myvar{1} \\  
& \hspace{0.6cm} \myvar{u}_{\text{min}} \leq \myvar{u}\leq  \myvar{u}_{\text{max}}
\end{split}
\end{align}
where $\apmat{A}_k$ and $\apvar{b}_k$ are the concatenations of the constraints \eqref{eq:no slip constraint}, \eqref{eq:joint position constraint}, and \eqref{eq:Linear u constraint r} evaluated at the sampling time $t = kT$, and $\apscalar{\nu}_h = \text{max}\{\apscalar{\nu}_q, \apscalar{\nu}_r \}$

Let the set of grasp constraint admissible control torques be:
\begin{equation}\label{eq:admissible control torques}
\robubset{S}_u = \robubset{S}_{u_f} \bigcap \robubset{S}_{u_q} \bigcap \robubset{S}_{u_r} \bigcap \ubscr{S}_{u_\tau}
\end{equation}
The following assumption is made to ensure that a solution to \eqref{eq:admissible control torques} exists as per Theorem \ref{thm:consat sampled data qp exact}:
\begin{assumption}\label{asm:safe control torques nonempty}
The set of grasp constraint admissible control torques $\robubset{S}_u$ is non-empty.
\end{assumption}

The following theorem guarantees forward invariance of $\mathscr{H}$ despite model uncertainties. For ease of notation, let $\myvar{c}$ be the concatenation of the robustness margin terms $\delta_{q\text{min}_j}$,$ \delta_{q\text{max}_j}$,$ \delta_{r_l}$, $\beta_{q\text{min}_j}$,$ \beta_{q\text{max}_j}$,$ \beta_{r_l}$ for all $j\in [1,m], l \in [1,4]$. 

\begin{thm}\label{thm:grasp consat robust}
Suppose Assumptions \ref{asm:square Jh}-\ref{asm:safe control torques nonempty} hold for the controllable system \eqref{eq:hand dynamics}. For twice continuously differentiable, extended class-$\mathcal{K}$ function $\alpha_1$, and locally Lipschitz extended class-$\mathcal{K}$ function $\alpha_2$, suppose that $\ubscr{C}_f$, $\robubset{B}_q \bigcap \robubset{C}_q$, $\robubset{B}_r \bigcap \robubset{C}_r$ are non-empty, and $\myvar{f}_c(0) \in \ubscr{C}_f$, $ (\myvar{q}(0), \myvardot{q}(0)) \in \robubset{B}_q \bigcap \robubset{C}_q$, $ \myvar{\xi}_f(0)$, $\myvardot{\xi}_f(0)) \in \robubset{B}_r \bigcap \robubset{C}_r$. Then there exists $\apscalar{\nu}_h$, $\varepsilon$, $\myvar{c}$ such that $\myvar{u}_k^*$ from \eqref{eq:safe control qp sampled data} applied to \eqref{eq:hand dynamics} ensures $(\myvar{f}_c, \myvar{q}, \myvar{\xi}_{cf})$ remains in $\mathscr{H}$ for $t \in [0, NT)$.
\end{thm}
\begin{proof}
By Assumptions \ref{asm:square Jh}-\ref{asm:geometric parameterization}, and for $(\myvar{f}_c, \myvar{q}, \myvar{\xi}_f) \in \mathscr{H}$, the constraints \eqref{eq:no slip constraint},  \eqref{eq:joint position constraint}, and \eqref{eq:Linear u constraint r} are well defined.  By Lemma \ref{lemma:contact force invariance}, there exists a $\varepsilon$ such that $\myvar{f}_c$ remains in $\ubscr{C}_f$ for all $t \in [0, NT)$. By Assumption \ref{asm:smooth surfaces} and construction of $h_{q_j}$, $h_{r_j}$, all terms in  \eqref{eq:no slip constraint},  \eqref{eq:joint position constraint}, \eqref{eq:Linear u constraint r}, \eqref{eq:actuator constraint} are locally Lipschitz continuous between sampling periods. Furthermore, by Assumption \ref{asm:safe control torques nonempty}, the constraint set is feasible for all $t\geq0$, and the sets $\ubscr{B}_q$, $\ubscr{C}_q$, $\ubscr{B}_r$, and $\ubscr{C}_r$ are compact by construction. Thus the conditions of Theorem \ref{thm:consat sampled data qp exact} are satisfied such that there exist $\apscalar{\nu}_h$, $\myvar{c}$, such that $\myvar{q}$ remains in $\ubscr{C}_q$, and $\myvar{\xi}_{cf}$ remains in $\ubscr{C}_r$ for all $t \in [0, NT)$, and the proof is complete.
\end{proof}

\section{Results} \label{sec: results}

The proposed zeroing control barrier functions have been presented with guarantees of forward invariance of the constraint set in the presence of sampling and external perturbations. The zeroing control barrier functions were then applied to address the challenging problem of robotic grasping to ensure that the object does not slip, the hand does not exceed joint limits, and the contact locations do not exceed the workspace of the fingertips. In this section, the proposed control is applied to robotic grasping and is implemented in simulation and hardware to demonstrate robust constraint satisfaction. 

Robust constraint satisfaction in this context refers to satisfaction of constraints with uncertain model parameters. This is to demonstrate the capability of the proposed method to a wide range of grasping scenarios in which intimate knowledge of the hand-object properties are unavailable. Here, tactile-based blind grasping is addressed in which the controller only has access to the contact location, $\myvar{p}_{fc}$, and joint angle positions/velocities, $\myvar{q}$, and $\myvardot{q}$. The remaining object parameters are approximated as follows. The approximate object center of mass, $\apvar{p}_o$ and orientation $\apvar{\gamma}_o$ are defined by the virtual frame \cite{ShawCortez2018b, Wimbock2012, Tahara2010}. The approximations $\apvar{p}_o$ and $\apvar{\gamma}_o$ are used to compute the approximations $\apmat{G}$, $\apmat{M}_o$, and $\apmat{C}_h$ terms. Nominal object parameters for object mass, $m_o$ and inertia $I_o$ are then arbitrarily defined, as will be discussed in the proceeding section.

The contact model is addressed by conservatively approximating the local object surface as flat. Note this approximation only holds if the fingertip surfaces are locally curved about the contact point. By approximating the object locally as flat, the contact kinematic parameters simplify to $\apmat{K}_{co_i} = 0_{2\times 2}, \apmat{M}_{co_i} = I_{2\times 2}, \apmat{T}_{co_i} = 0_{1\times 2}$. This not only  simplifies the computation of $\apmat{A}_r$, $\apvar{b}_r$, but also introduces robustness because flat object surfaces result in the largest contact point displacement for the same given rolling angular velocity.

\subsection{Simulation Results}\label{ssec:sim results}

The purpose of this simulation is to demonstrate constraint satisfaction via the proposed constraint satisfying controller. These results will demonstrate that the proposed control \eqref{eq:safe control qp sampled data} can ensure grasp constraint satisfaction for a given nominal control despite model uncertainties. 

The hand used in the simulations is a nine degree of freedom, fully-actuated hand with three identical fingers and hemispherical fingertips. The maximum control torque is $3.5$ Nm. The object being grasped is a cube, and the model parameters of the hand-object are found in Table \ref{table:Simulation parameters} with the initial hand-object configuration shown in Figure \ref{fig:initgrasp}.

The nominal manipulation control used here is the manipulation control presented from \cite{ShawCortez2018b} with the internal force control \cite{Kawamura2013} which is re-defined here as:
\begin{equation}\label{eq:nominal control tbbg}
\myvar{u}_{nom} = \apmat{J}_h^T  \Big( (P^T \apmat{G})^\dagger ( -K_p \myvar{e} - K_i \text{sat}(\int_0^t \myvar{e} \ dt)  - K_d \myvardot{e} )  + \myvar{u}_f \Big)
\end{equation}
\begin{equation} \label{eq:nominal internal force control tbbg}
\bm{u}_f = k_f (\bmb{p}_c - \bm{p}_{c_1},  \bmb{p}_c- \bm{p}_{c_2}, ...,  \bmb{p}_c - \bm{p}_{c_n} )
\end{equation}
\begin{equation}
\text{sat}(\myvar{x})_j = \left\{\begin{array}{lr}
        x_j, & \text{for } |x_j| \leq 3 \\
        3 \text{sign}(x_j), & \text{for } |x_j| > 3
        \end{array}\right\} 
\end{equation}
Note the saturation function is used to anticipate integrator wind-up should the proposed control $\myvar{u}_k^*$ from \eqref{eq:safe control qp sampled data} diverge from $\myvar{u}_{nom}$.  

Implementation of the contact location constraints requires appropriate parameterizations of the fingertip surface to satisfy Assumption \ref{asm:geometric parameterization}. The parameterization used here is $\myvar{c}_{f_i} = [-R \cos(a_{f_i}) \cos(b_{f_i}), R \sin(a_{f_i}), -R \cos(a_{f_i})\sin(b_{f_i})]^T$. The associated box constraints to define the fingertip workspace are: $-\pi/2 < a_{f_i} < \pi/2$, $-\pi < b_{f_i} < 0$. The joint angle limits for each finger are $\myvar{q}_{\text{max}_i} = (3 \pi /4, \pi/3, 3 \pi/4)$, $\myvar{q}_{\text{min}_i} = (0, -\pi/3, 0)$, for $ i \in [1,3]$. 

\begin{figure}[h!]
\centering
	\subcaptionbox{Isometric view. \label{fig:initgraspiso}}
		{\includegraphics[scale=.2]{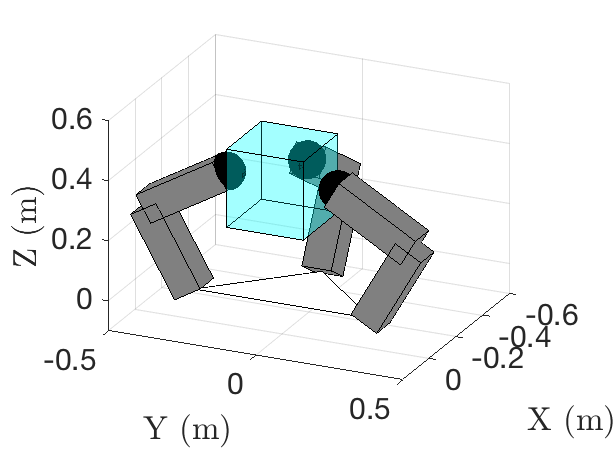}}
	\subcaptionbox{Top view. \label{fig:initgraspside} }
		{\includegraphics[scale=.2]{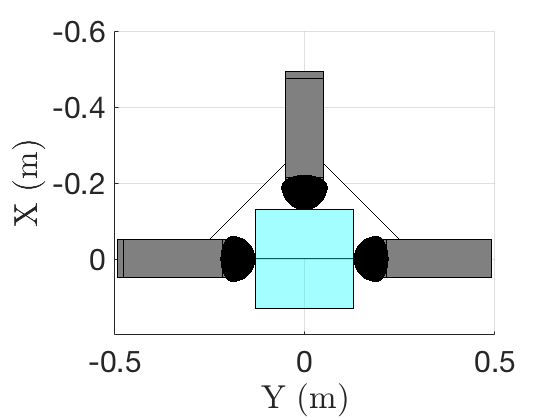}}
	\caption{Simulation setup with initial hand-object configuration.} \label{fig:initgrasp}
\end{figure}

\begin{table}[h!]\hspace*{-1cm}
\centering
\caption{Simulation Parameters} \label{table:Simulation parameters}
\begin{tabular}{c|c}    \toprule
Link dimensions    & $0.05 \ \text{m} \times 0.05 \ \text{m} \times 0.3\  \text{m}$  \\
Link mass & $0.25 \ \text{kg}$ \\ 
Link moment of inertia & $\text{diag}([0.0019,0.0001,0.0019]) \ \text{kg}\text{m}^2$  \\
Fingertip radius & $0.06 \ \text{m}$  \\
Object dimensions & $0.260 \ \text{m} \times 0.260 \ \text{m} \times 0.260 \ \text{m}$ \\
Object mass & $0.11 \ \text{kg}$\\
Object moment of inertia & $\text{diag}([0.0058, 0.0214, 0.0214])\ \text{kg} \text{m}^2$ \\
Friction coefficient & $\mu = 0.9$ \\
Initial $\myvar{p}_o$ & $[0.00,0.00,0.41] \ \text{m}$  \\
Initial $\myvar{\gamma}_o$ & $[0.0,0.0,0.0] \ \text{rad}$  \\\bottomrule
\end{tabular}
\end{table}

The control gains used in the simulation are $K_p = 0.26 I_{6\times 6}$, $K_i = 0.1 I_{6\times 6}$,  $K_d = 0.125 I_{6\times 6}$, and $k_f = 1.0$. The set-point object reference command is $\myvar{r} = \myvar{x}(0) +  (0,0,0,0,0,\pi/2)$, where $\myvar{x} \in \mathbb{R}^6$ is the task state defined by the virtual frame \cite{ShawCortez2018b}. The proposed constraint satisfying control \eqref{eq:safe control qp sampled data} is implemented with the following robustness margins: $\varepsilon = 0.03$, $\delta_{r_l} = 0.1$ rad, $\beta_{r_l} = 1.0$ rad/s, $\delta_{q_i} = 0.1$ rad, $\beta_{q_i} = 0.05$ rad/s. The extended class-$\mathcal{K}$ functions used in this simulation were $\alpha_1(h) = 2h$ and $\alpha_2(h) = h^3$. The approximate object model parameters are purposefully offset from the true mass parameters. The object mass error was set to $\Delta(m)_o = 0.1$ kg, inertia error set to $\Delta(I)_o = 0.001 I_{3\times 3} \ \text{kg m}^2$. The sampling time margin was set to $\apscalar{\nu}_h = 0.0001$. A four-sided pyramid was used to approximate the friction cone with an associated friction coefficient of $\apscalar{\mu} = 0.64$. The simulations were implemented in Matlab 2018 with a sampling time of $T = 0.003$ s.

Figure \ref{fig:exp_42_results} shows the results of the nominal control as it attempts to track the set-point reference command by twisting the object about the $Z$-axis. The plots show multiple constraint violations including slip, joint over-extension, and excessive rolling, which result in a failed grasp. At $t = 0.928$ s, the contact location $b_{cf_2}$ exceeded the fingertip surface and the simulation stopped. Figures \ref{fig:exp_42_trackingperf} and \ref{fig:fail_grasp_exp42} show the reference error for the $Z$ component of the orientation error and final grasp configuration. It is important to emphasize that the final orientation of $1.394$ rad shown in Figure \ref{fig:exp_42_trackingperf} is not feasible due to slip (see Figure \ref{fig:frictionexp42}), singular configurations (see Figure \ref{fig:joint_traj_exp42_q3} at $t = 0.461$ s) and joint limits that were exceeded (see Figure \ref{fig:joint_traj_exp42_q2} at $t = 0.498$ s) during the manipulation motion, prior to the excessive rolling that stopped the simulation. These results demonstrate that the conventional assumption that the grasp conditions hold are not valid and may result in grasp failure despite using a stable manipulation controller.

\begin{figure}[!t]
\centering
	\subcaptionbox{ \label{fig:contact_traj_exp42_bfi} }
		{\includegraphics[scale=.221]{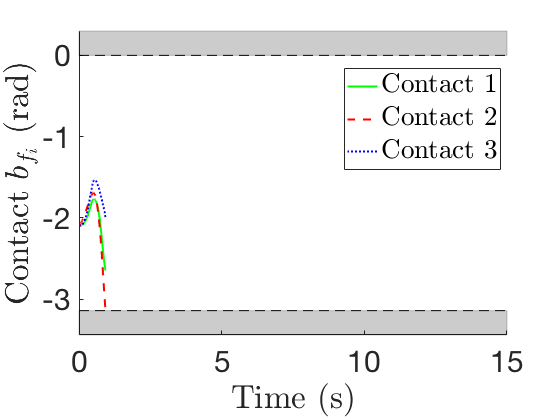}}
	\subcaptionbox{\label{fig:joint_traj_exp42_q2} }
		{\includegraphics[scale=.221]{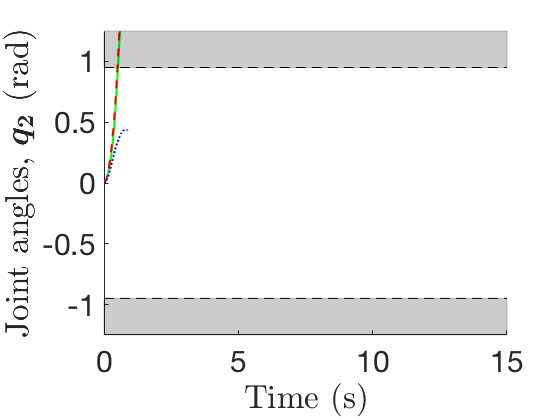}}
		\subcaptionbox{\label{fig:joint_traj_exp42_q3} }
		{\includegraphics[scale=.221]{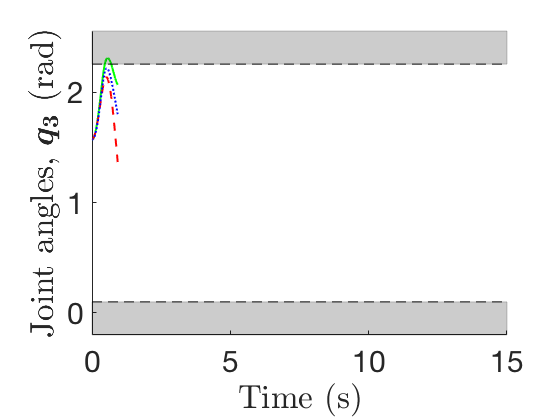}}
		\subcaptionbox{\label{fig:frictionexp42} }
		{\includegraphics[scale=.2]{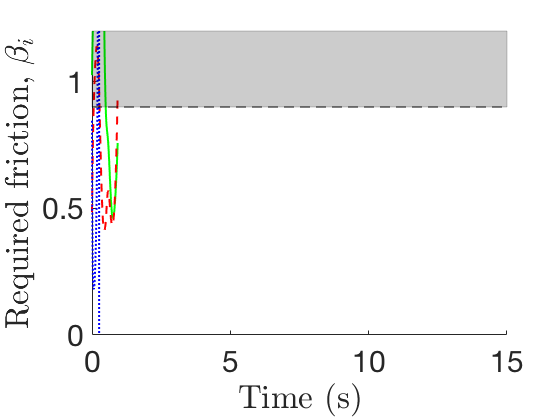}}
		\subcaptionbox{\label{fig:exp_42_trackingperf} }
		{\includegraphics[scale=.2]{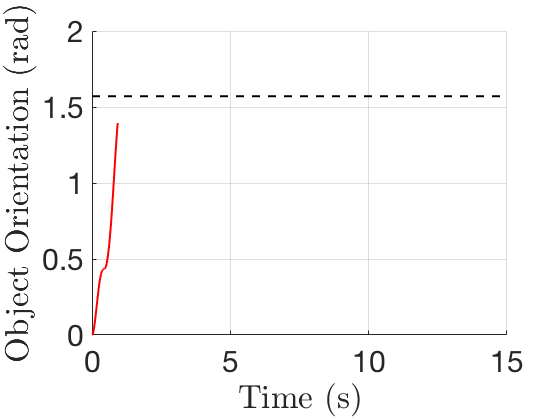}}
		\subcaptionbox{\label{fig:fail_grasp_exp42} }
		{\includegraphics[scale=.221]{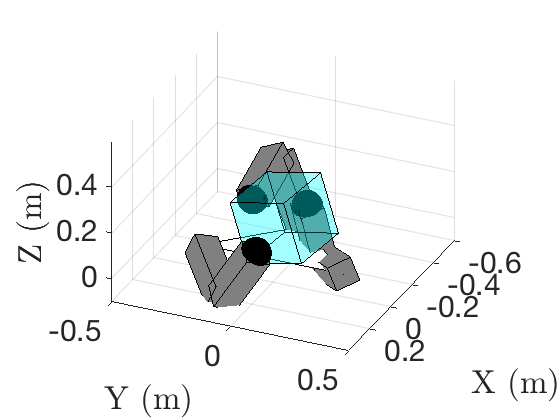}}
	\caption{Failed grasp for nominal without constraint satisfaction. The simulation is stopped when $\myvar{b}_{f_i}$ exceeds the constraint boundary.  Note, (e) shows the Z-component of object orientation. The black dashed line corresponds to the reference, $\myvar{r}$, and the red line corresponds to the state $\myvar{x}$.} \label{fig:exp_42_results}
\end{figure}

Figures \ref{fig:exp_41_results} shows the nominal control implemented with the proposed control \eqref{eq:safe control qp sampled data}. The resulting plots show that the proposed control is able to ensure the grasp states remain in $\mathscr{H}$ despite the model uncertainty from the tactile-based blind grasping implementation. Figure \ref{fig:exp_41_trackingperf} shows the resulting tracking error of the proposed control. To prevent grasp failure, the proposed control prevents the hand-object system from reaching the infeasible reference command. 
%

\begin{figure}[hbtp]
\centering
	\subcaptionbox{ \label{fig:contact_traj_exp41_bfi} }
		{\includegraphics[scale=.221]{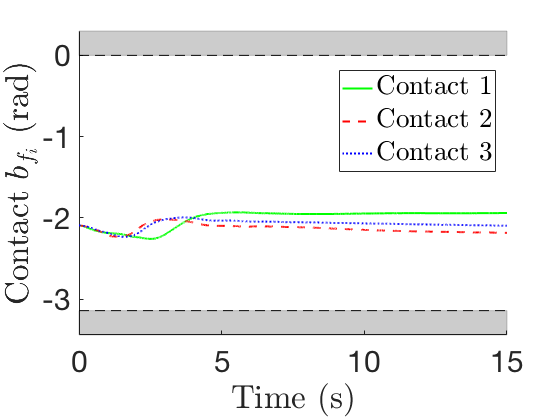}}
	\subcaptionbox{\label{fig:joint_traj_exp41_q2} }
		{\includegraphics[scale=.221]{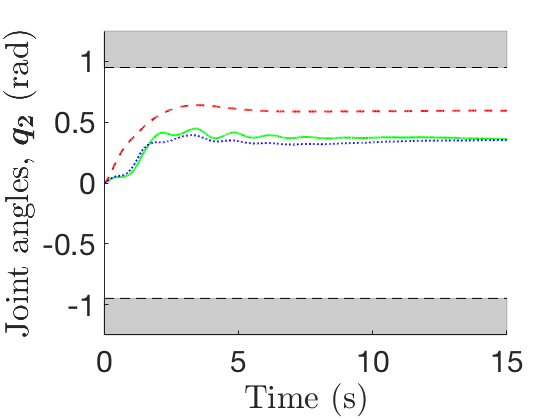}}
		\subcaptionbox{\label{fig:joint_traj_exp41_q3} }
		{\includegraphics[scale=.221]{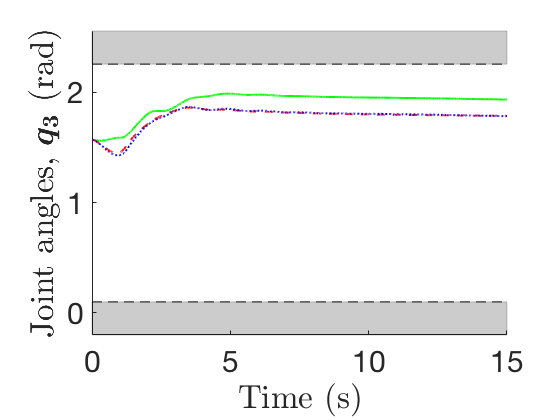}}
		\subcaptionbox{\label{fig:frictionexp41} }
		{\includegraphics[scale=.20]{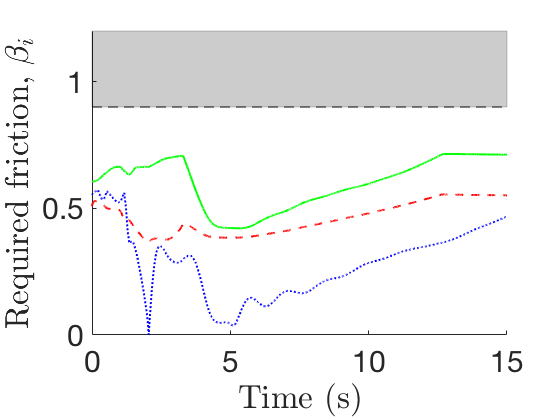}}
		\subcaptionbox{\label{fig:exp_41_trackingperf} }
		{\includegraphics[scale=.221]{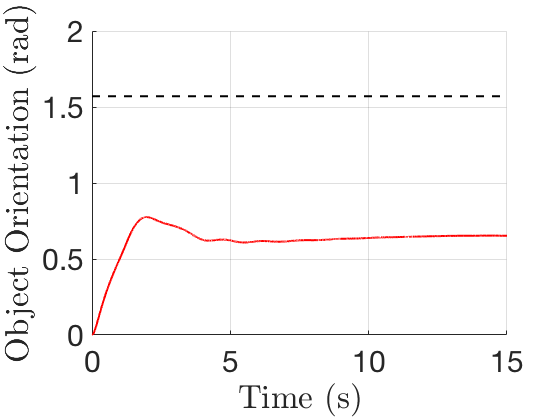}}
		\subcaptionbox{Final grasp configuration. \label{fig:final_grasp_exp41} }
		{\includegraphics[scale=.221]{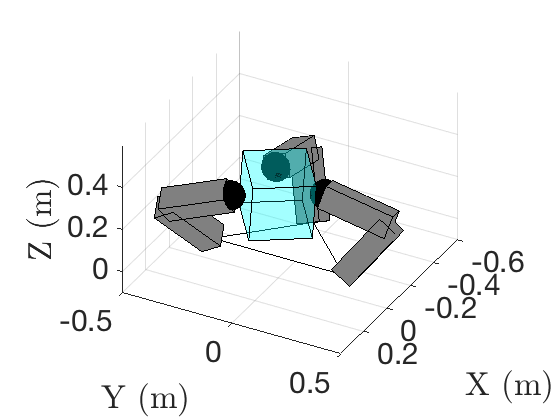}}
	\caption{Successful grasp for constraint satisfying controller. Note, (e) shows the Z-component of object orientation. The black dashed line corresponds to the reference, $\myvar{r}$, and the red line corresponds to the state $\myvar{x}$. } \label{fig:exp_41_results}
\end{figure}

\begin{remark}
One important note is that the proposed control is effectively a disturbance on the nominal control input when the proposed control must intervene to satisfy the grasp constraints. Although the proposed control ensures grasp constraint satisfaction, there is yet no analysis of how the nominal control will behave. For implementation, it may be advantageous for the nominal controller to be passive to avoid undesired manipulation motion. Fortunately, many tactile-based blind grasping controllers are passivity-based \cite{Ozawa2017}. Further analysis is required to address passivity of the proposed control \eqref{eq:safe control qp sampled data}.
\end{remark}

 \subsection{Hardware Results}

The purpose of the hardware implementation is to demonstrate the effectiveness of the proposed control in practice. This is done by performing three demonstrations in which the nominal manipulation control used in Section \ref{ssec:sim results} is implemented along with the proposed control \eqref{eq:safe control qp sampled data}. In the first demonstration, a feasible reference is provided to the proposed control to show that when no constraint violation occurs, the proposed control admits the original nominal control law. In the second demonstration, a compromising reference is provided to the nominal control to show that when the grasp constraints are not formally addressed, instabilities may occur in the hand-object system that result in grasp failure. In the final demonstration, the \textit{same} compromising reference is provided to the proposed control to show how the proposed method ensures constraint satisfaction. 

\begin{figure}[hbtp]
\centering
	\subcaptionbox{Top View \label{fig.Allegrohandsetup1} }
		{\includegraphics[scale=.23]{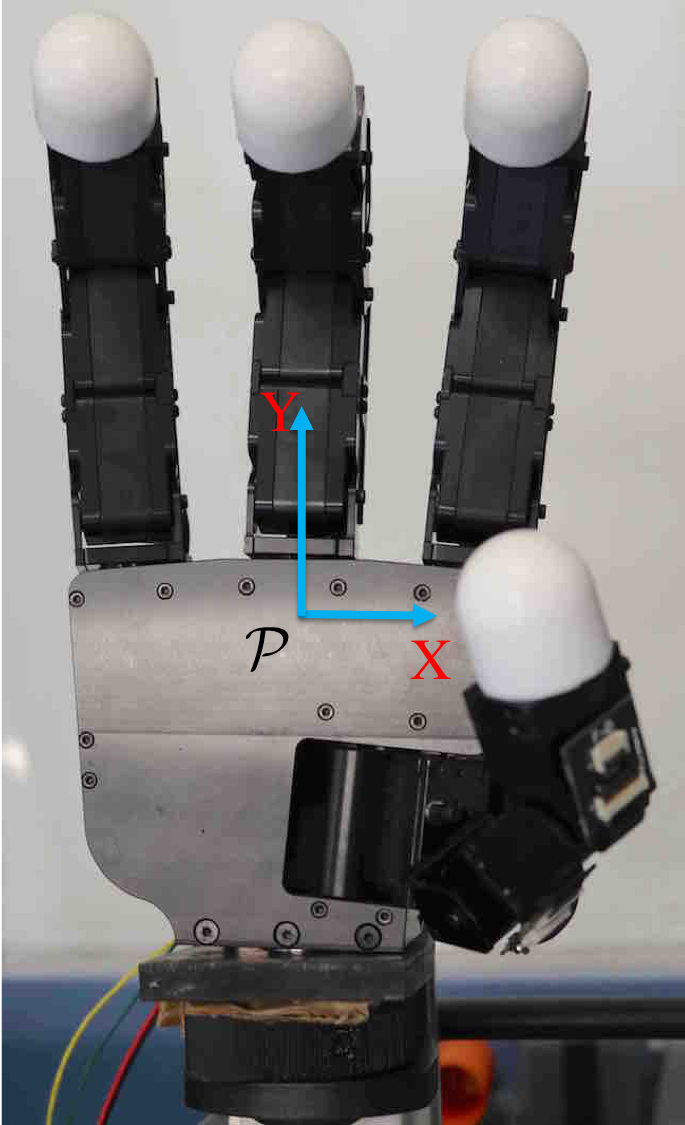}}
	\subcaptionbox{ Side View \label{fig.Allegrohandsetup2} }
		{\includegraphics[scale=.237]{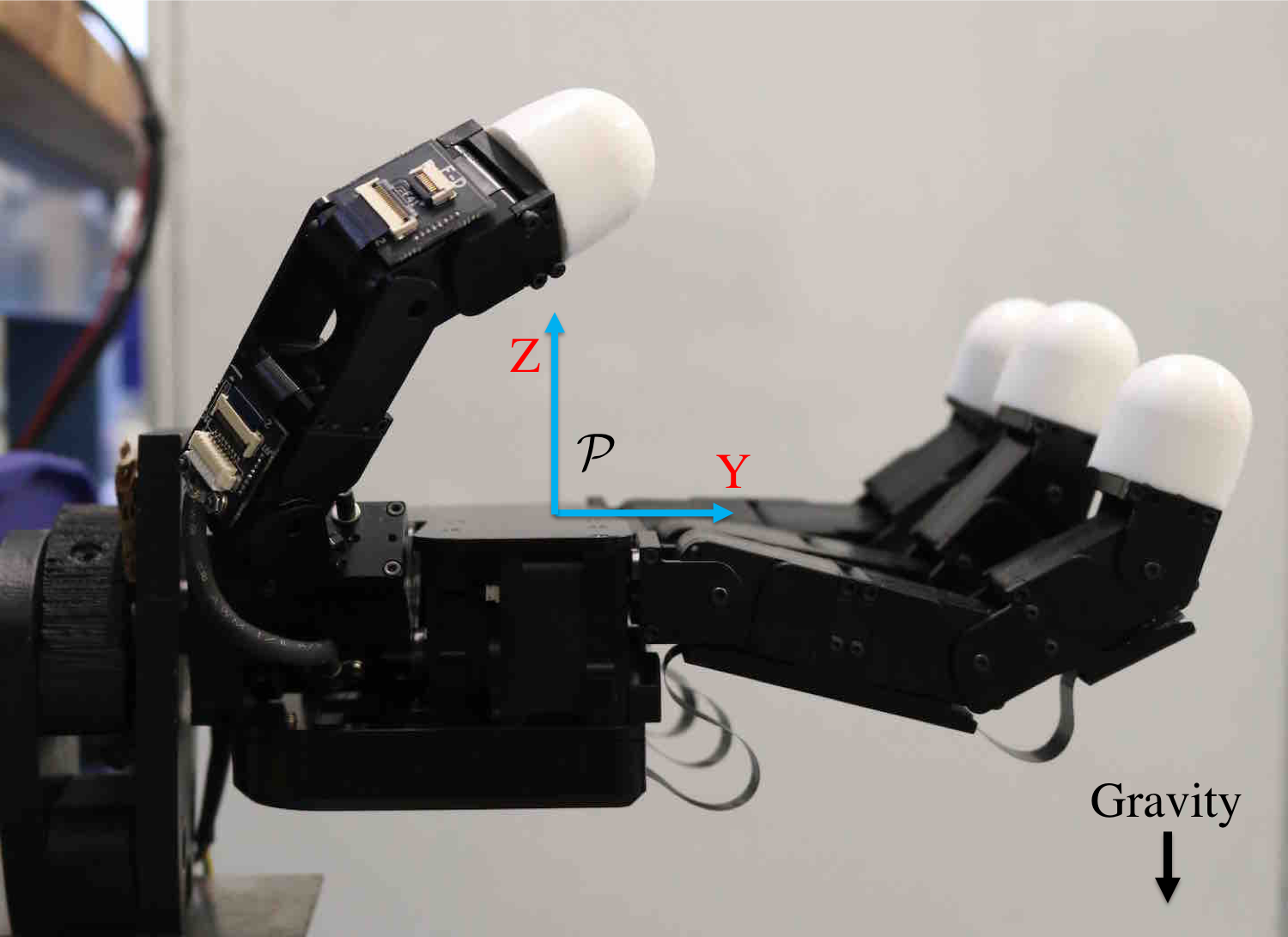}}
	\caption{Allegro Hand setup.}  \label{fig.AllegroHandsetup}
\end{figure}

\begin{table}[htbp]\hspace*{-1cm}
\centering
\caption{Allegro Hand: Model Parameters} \label{table:allegro hand parameters}
\begin{tabular}{c|ccc}    \toprule
 \textbf{Dimensions} (m) & \emph{Link 1} & \emph{Link 2} & \emph{Link 3} \\\midrule
Length (index, middle, ring)    & $0.0540$ & $0.0384$ &  $0.0250$  \\
Length (thumb)   & $ 0.0554 $ & $0.0514 $ & $0.0400$  \\
Width/Depth (all) & $0.0196$ & $0.0196$ & $0.0196$\\
\midrule
\textbf{Mass} (kg) & \emph{Link 1} & \emph{Link 2} & \emph{Link 3} \\\midrule
(index, middle, ring) & $0.0444 $ & $0.0325$ & $0.0619$  \\
 (thumb) & $0.0176$ & $0.0499$ & $0.0556$  \\
\bottomrule
\end{tabular}
\end{table}
 
\begin{figure}[ht]
\centering
	\subcaptionbox{Initial configuration \label{fig:demo 1 initial config} }
		{\includegraphics[scale=.292]{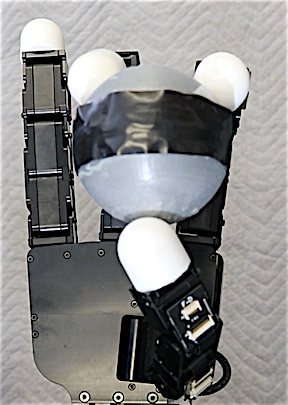}}
	\subcaptionbox{ Final configuration \label{fig:demo 1 final config} }
		{\includegraphics[scale=.281]{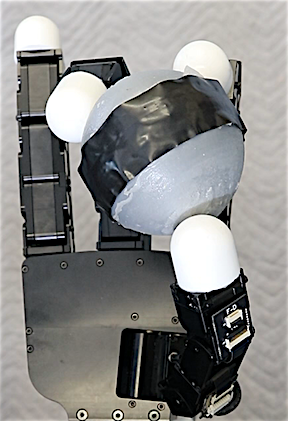}}
	\subcaptionbox{Orientation tracking.  \label{fig:demo 1_orientation} }
		{\includegraphics[scale=.2]{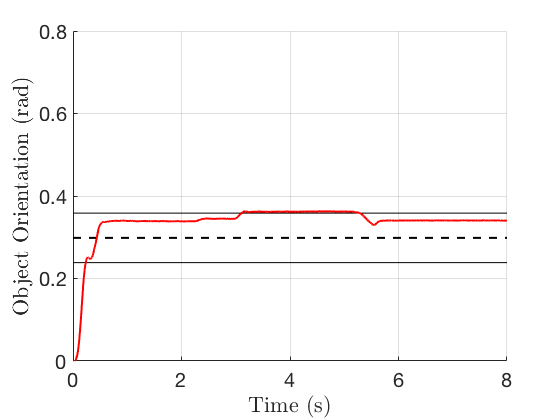}}
	\subcaptionbox{Nominal control torque \label{fig:demo 1_nom torque} }
		{\includegraphics[scale=.2]{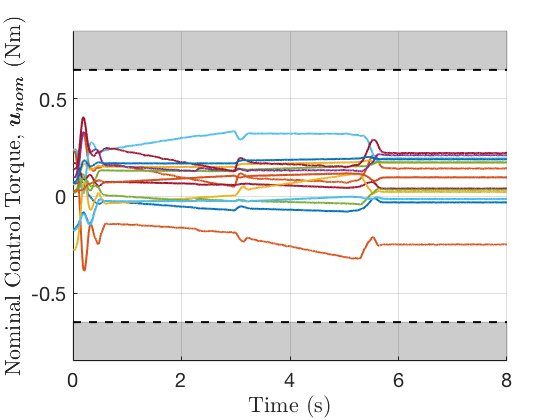}}
		\subcaptionbox{Proposed control torque \label{fig:demo 1_consat torque} }
		{\includegraphics[scale=.2]{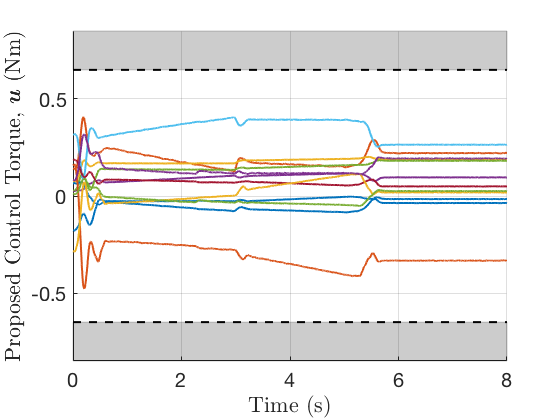}}
	\caption{Demonstration 1: Proposed control with feasible reference $r_\psi = 0.3$ rad.}  \label{fig:demo 1}
\end{figure}

The Allegro Hand is used to implement the controllers and is depicted in Figure \ref{fig.AllegroHandsetup}. The Allegro Hand is a fully-actuated hand with 16 degrees of freedom, and with encoders at each joint to provide measurements of $\myvar{q}$. The maximum torque output of each motor is $0.65$ Nm. Tactile sensors are emulated to provide approximate contact measurements of $\myvar{\xi}_{cf}$ to the proposed controller. Again, the implementation of the proposed control is as a tactile-based blind grasping scheme in which only measurements of $\myvar{q}$, $\myvar{\xi}_{cf}$ are available. Thus robustness in this context refers to uncertainty in the object model and state. The experimental setup includes a NI USB-8473s High-Speed CAN with a fixed sampling time of $T_s = 0.003 s$. The model parameters used for the Allegro Hand are shown in Table \ref{table:allegro hand parameters}.

The nominal object used to implement the proposed control is a cube of side lengths $0.04$ m and mass of $\apscalar{m}_o = 0.05$ kg. The true object used in the experiments is a sphere of radius $0.0375$ m and mass of $0.09$ kg. Note the difference in mass and shape also correspond to discrepancies in the object inertia.

The nominal control gains used in the experiment are  $K_p$ = diag(500, 500, 500, 0.8, 0.8, 0.8), $K_i = $ diag(50, 50, 50, 0.6, 0.6, 0.6),  $K_d =$ diag(0.008, 0.008, 0.008, 0.16, 0.16, 0.16), and $k_f = 60$. The set-point object reference command is $\myvar{r} = \myvar{x}(0) +  (0,0,0,0,0,r_\psi)$, where $r_\psi \in \mathbb{R}$ and the same virtual frame from \cite{ShawCortez2018b} is used to define $\myvar{x}$. The proposed constraint satisfying control \eqref{eq:safe control qp sampled data} is implemented with the following robustness margins: $\apscalar{\nu}_h = 0.01$, $\varepsilon = 0.15$, $\delta_{r_j} = 0.10$ rad, $\beta_{r_j} = 0.10$ rad/s, $\delta_{q_j} = 0.05$ rad, $\beta_{q_j} = 0.10$ rad/s, $j \in [1,l]$. The extended class-$\mathcal{K}$ functions used were $\alpha_1(h) = 3.3 h$ and $\alpha_2(h) = 10 h^3$. A four-sided pyramid was used to approximate the friction cone with associated friction coefficient of $\apscalar{\mu} = 1.06$. The same contact parameterizations from Section \ref{ssec:sim results} were used here for the hemispherical fingertips. Note, in the following figures, the gray regions depict the area outside of the constraint admissible set.

Figure \ref{fig:demo 1} shows the results of the first demonstration in which the proposed control \eqref{eq:safe control qp sampled data} is implemented with the nominal control for the feasible reference, $r_\psi = 0.3\pm 0.06$ rad. Figure \ref{fig:demo 1_orientation} shows the $\psi$ component of the state $\myvar{x}$ reach within the reference tolerance for a successful manipulation, and the final configuration is depicted in Figure \ref{fig:demo 1 final config}. The proposed control torque (see Figure \ref{fig:demo 1_consat torque} closely matches the nominal control torque (see Figure \ref{fig:demo 1_nom torque}). This demonstration illustrates how the proposed control admits the nominal controller to achieve the desired manipulation motion with limited interference, and successful manipulation. The following demonstration investigates the use of the nominal control alone to reach a compromising reference command.

\begin{figure}[t]
\centering
	\subcaptionbox{Initial configuration \label{fig:demo 2 initial config} }
		{\includegraphics[scale=.242]{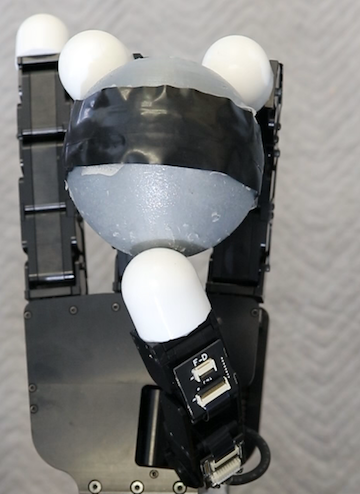}}
	\subcaptionbox{Unstable configuration \label{fig:demo 2 final config} }
		{\includegraphics[scale=.2645]{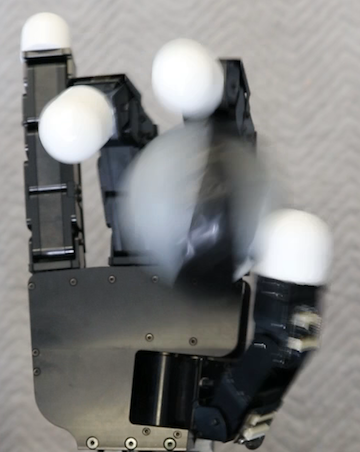}}
	\subcaptionbox{Orientation tracking.  \label{fig:demo 2_orientation} }
		{\includegraphics[scale=.221]{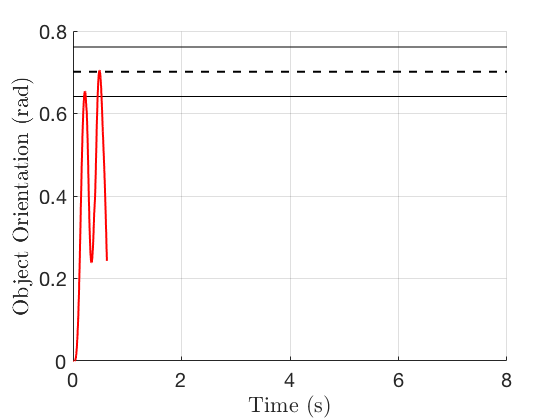}}
	\subcaptionbox{Nominal control torque \label{fig:demo 2_nom torque} }
		{\includegraphics[scale=.221]{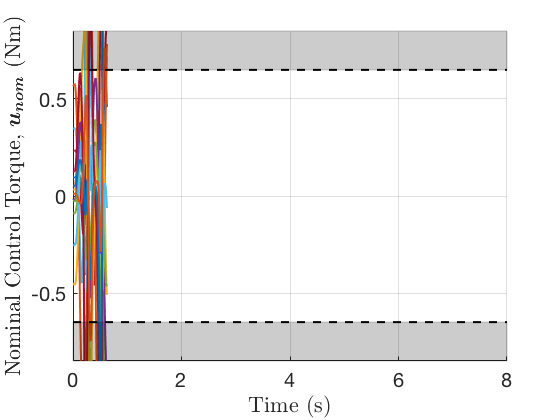}}
	\caption{Demonstration 2: Nominal control \textit{only} with compromising reference $r_\psi = 0.7$ rad.}  \label{fig:demo 2}
\end{figure}

In the second demonstration, shown in Figure \ref{fig:demo 2}, the nominal control \textit{alone} is implemented for the compromising reference $r_\psi = 0.7 \pm 0.06$ rad. The plots clearly depict the unstable behavior of the system as the nominal control attempts to reach the reference. Figure \ref{fig:demo 2 final config} shows the unstable configuration of the hand that results in loss of contact and grasp failure. Note, Figure \ref{fig:demo 2_nom torque} shows the nominal control exceeding the actuation capabilities of the hand. This demonstration shows that when no proposed control is implemented, the nominal control is subject to instabilities and ultimately grasp failure for a compromising reference. The final demonstration will investigate how the proposed control compensates for this compromising reference.

\begin{figure}[ht]
\centering
	\subcaptionbox{Initial configuration \label{fig:demo 3 initial config} }
		{\includegraphics[scale=.2997]{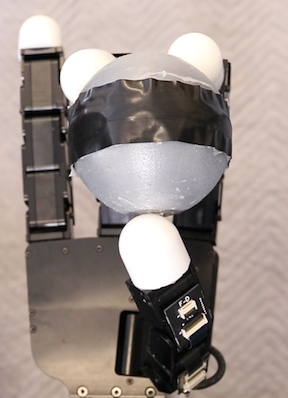}}
	\subcaptionbox{ Final configuration \label{fig:demo 3 final config} }
		{\includegraphics[scale=.25]{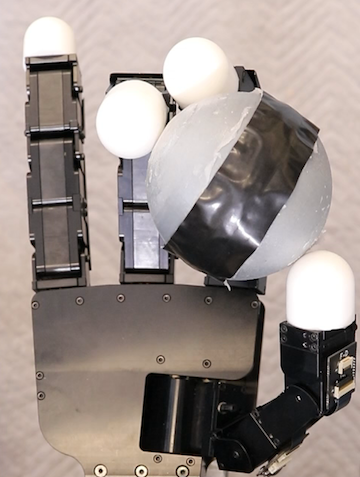}}
	\subcaptionbox{Orientation tracking.  \label{fig:demo 3_orientation} }
		{\includegraphics[scale=.221]{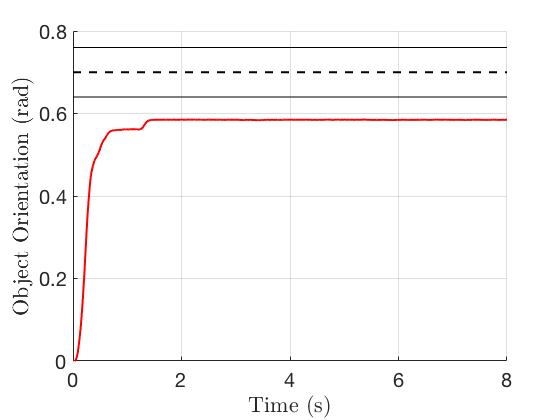}}
		\subcaptionbox{Contact location trajectories \label{fig:demo 3_acfi} }
		{\includegraphics[scale=.221]{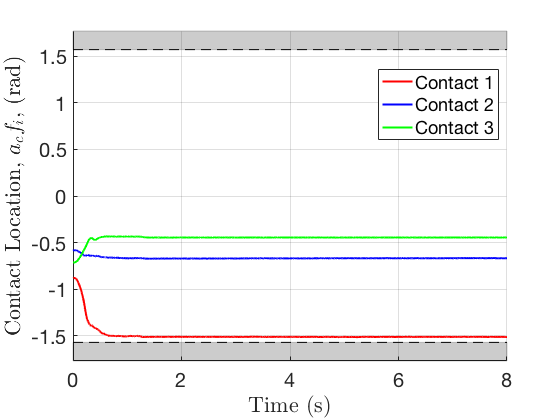}}
	\subcaptionbox{Nominal control torque \label{fig:demo 3_nom torque} }
		{\includegraphics[scale=.221]{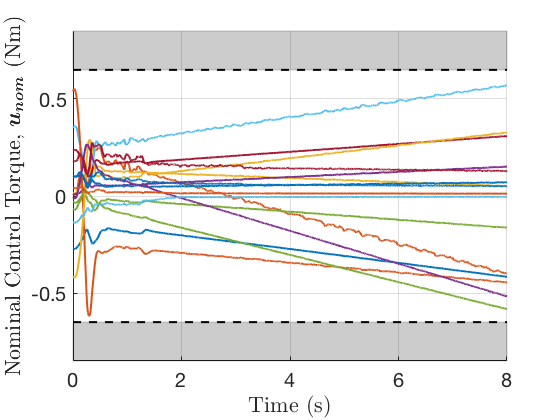}}
	\subcaptionbox{Proposed control torque \label{fig:demo 3_consat torque} }
		{\includegraphics[scale=.221]{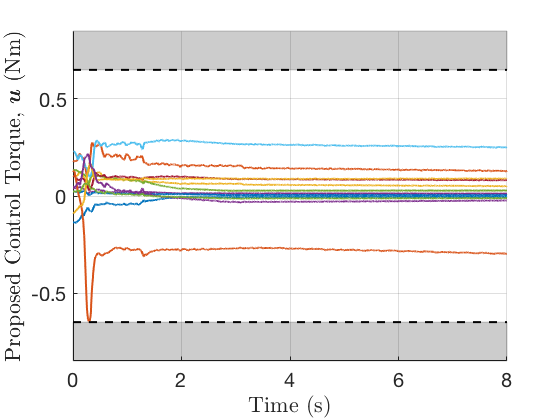}}
	
	\caption{Demonstration 3: Proposed control with compromising reference $r_\psi = 0.7$ rad.}  \label{fig:demo 3}
\end{figure}

Figure \ref{fig:demo 3} shows the results of the final demonstration in which the proposed control \eqref{eq:safe control qp sampled data} is implemented for the \textit{same} compromising reference of $r_\psi = 0.7 \pm 0.06$ rad. Figure \ref{fig:demo 3_orientation} shows the $\psi$ component of $\myvar{x}$ reaching a steady-state value outside of the reference tolerance, with the final configuration shown in Figure \ref{fig:demo 3 final config}. The reason for this steady-state offset is that the proposed control prioritizes constraint violation over implementation of the nominal control. Figures \ref{fig:demo 3_nom torque} and \ref{fig:demo 3_consat torque} show the deviation between nominal and proposed control torque as the proposed control intervenes to ensure constraint satisfaction. Figure \ref{fig:demo 3_acfi} shows the trajectory of $\myvar{a}_{cf_i}$, an element of $\myvar{\xi}_{cf}$, as $\myvar{a}_{cf_1}$ approaches the constraint boundary of $\ubscr{C}_r$. The plots show that the proposed control prevents $\myvar{\xi}_{cf}$ from exceeding the constraint set $\ubscr{C}_r$ to enforce grasp constraint satisfaction in the presence of model uncertainty and sampling time effects.

\section{Conclusion} \label{sec: conclusion}

In this paper, a novel control barrier function formulation was developed to ensure constraint satisfaction for mechanical systems. The proposed method is robust to external perturbations and sampling time effects, and provides a systematic method of bounding the system's velocity near the constraint boundary. The proposed method was then applied to the challenging problem of robotic grasping. A novel controller was proposed to ensure no slip, no over-extension (e.g. singular configurations), and no excessive rolling occurs while admitting an existing controller from the literature. The proposed control was implemented in simulation and hardware to demonstrate the efficacy of the proposed approach.

\section{Appendix}

The joint constraint-related terms $A_q$ and $b_q$ are:
\begin{equation}
A_q = \bracketmat{c}{I_{m\times m} \\ -I_{m\times m}}  \apmat{M}_h^{-1}( I_{m\times m} - \apmat{J}_h^T\apmat{B}_{ho}^{-1}\apmat{J}_h \apmat{M}_h^{-1} )
\end{equation}
\begin{equation}
\myvar{b}_q = \bracketmat{cccccc}{\myvar{b}_{q\text{min}_1} & \hdots & \myvar{b}_{q\text{min}_m} & \myvar{b}_{q\text{max}_1} & \hdots & \myvar{b}_{q\text{max}_m}}^T
\end{equation}
where
\begin{multline}
\myvar{b}_{q\text{min}_j} = - \myvar{e}_j \apmat{M}_h^{-1} \bigg( -\apmat{C}_h \myvardot{q} - \apmat{J}_h^T \apmat{B}_{ho}^{-1}  \Big( \apmat{J}_h \apmat{M}_h^{-1} ( -\apmat{C}_h \myvardot{q} + \apvar{\tau}_e) \\
+ \apmatdot{J}_h \myvardot{q} -\apmatdot{G}^T \apvardot{x}_o + \apmat{G}^T \apmat{M}_o^{-1}( \apmat{C}_o \apvardot{x}_o  - \apvar{w}_e) \Big) + \apvar{\tau}_e \bigg) \\
- \frac{\partial \alpha_1}{\partial \hat{h}_{\text{qmin}_j}} \hdot{h}_{q\text{min}_j} - \alpha_2( \hat{B}_{q\text{min}_j}) 
\end{multline}
\begin{multline}
\myvar{b}_{q\text{max}_j} =  \myvar{e}_j \apmat{M}_h^{-1} \bigg( -\apmat{C}_h \myvardot{q} - \apmat{J}_h^T \apmat{B}_{ho}^{-1}  \Big( \apmat{J}_h \apmat{M}_h^{-1} ( -\apmat{C}_h \myvardot{q} + \apvar{\tau}_e) \\
+ \apmatdot{J}_h \myvardot{q} -\apmatdot{G}^T \apvardot{x}_o + \apmat{G}^T \apmat{M}_o^{-1}( \apmat{C}_o \apvardot{x}_o  - \apvar{w}_e) \Big) + \apvar{\tau}_e \bigg) \\
- \frac{\partial \alpha_1}{\partial \hat{h}_{\text{qmax}_j}} \hdot{h}_{q\text{max}_j} - \alpha_2( \hat{B}_{q\text{max}_j}) 
\end{multline}

The contact location constraint-related terms $A_r$ and $b_r$ are:
\begin{equation}
A_r = \bracketmat{ccc}{A_{r_1}^T & \hdots & A_{r_n}^T}^T
\end{equation}
where
\begin{multline}
A_{r_i} = \bracketmat{cc}{1 & 0 \\ -1 & 0 \\ 0 & 1 \\ 0 & -1} \apmat{H} R_{c_i p} \bracketmat{cc}{0_{3\times 3} & I_{3\times 3}} \Big( \apmat{J}_{s_i} E_i \\
- (\apmat{J}_{s_i}E_i \apmat{M}_h^{-1} \apmat{J}_h^T + \apmat{M}_o^{-1} \apmat{G}) \apmat{B}_{ho}^{-1} \apmat{J}_h \Big) \apmat{M}_h^{-1} 
\end{multline}

\begin{equation}
\myvar{b}_r = \bracketmat{ccc}{\myvar{b}_{r_1} & \hdots & \myvar{b}_{r_n}}^T
\end{equation}

\begin{multline}
\myvar{b}_{r_i} = - \bracketmat{cc}{1 & 0 \\ -1 & 0 \\ 0 & 1 \\ 0 & -1} \Bigg(  \Big( \apmatdot{H} R_{c_i p} + \apmat{H} \dot{R}_{c_i p} \Big) (\apvar{\omega}_{f_i} - \apvar{\omega}_o)\\
+ \apmat{H} R_{c_i p}\bracketmat{cc}{0_{3\times 3} & I_{3\times 3}}  \Bigg( \apmatdot{J}_{s_i} \myvardot{q}_i + \apmat{J}_{s_i} E_i \apmat{M}_h^{-1} \bigg( -\apmat{C}_h \myvardot{q}\\
 - \apmat{J}_h^T \apmat{B}_{ho}^{-1}  \Big( \apmat{J}_h \apmat{M}_h^{-1} ( -\apmat{C}_h \myvardot{q} + \apvar{\tau}_e)+ \apmatdot{J}_h \myvardot{q} - \apmatdot{G}^T \apvardot{x}_o \\
 + \apmat{G}^T \apmat{M}_o^{-1}( \apmat{C}_o \apvardot{x}_o - \apvar{w}_e) \Big) + \apvar{\tau}_e  \bigg) - \apmat{M}_o^{-1} \bigg( -\apmat{C}_o \apvardot{x}_o \\
 + \apmat{G} \apmat{B}_{ho}^{-1}  \Big( \apmat{J}_h \apmat{M}_h^{-1} ( -\apmat{C}_h \myvardot{q}  + \apvar{\tau}_e) + \apmatdot{J}_h \myvardot{q}   -\apmatdot{G}^T  \apvardot{x}_o \\
 + \apmat{G}^T \apmat{M}_o^{-1}( \apmat{C}_o \apvardot{x}_o - \apvar{w}_e) \Big)+ \apvar{w}_e \bigg)    \Bigg) \Bigg) \\
 - \bracketmat{c}{ \frac{\partial \alpha_1}{\partial \hat{h}_{r_1}} \hdot{h}_{r_1} + \alpha_2(\hat{B}_{r_1}) \\ \frac{\partial \alpha_1}{\partial \hat{h}_{r_2}} \hdot{h}_{r_2} +  \alpha_2(\hat{B}_{r_2}) \\ \frac{\partial \alpha_1}{\partial \hat{h}_{r_3}} \hdot{h}_{r_3} + \alpha_2(\hat{B}_{r_3}) \\\frac{\partial \alpha_1}{\partial \hat{h}_{r_4}} \hdot{h}_{r_4} + \alpha_2(\hat{B}_{r_4})  } 
\end{multline}

\bibliographystyle{IEEEtran}
\bibliography{IEEEabrv,ShawCortez_Journal_consat}

\end{document}